\documentclass[11pt]{article}

\usepackage[T1]{fontenc}
\usepackage{lmodern}
\usepackage[utf8]{inputenc} 
\usepackage{url}            
\usepackage{booktabs}       
\usepackage{nicefrac}       
\usepackage{microtype}      
\usepackage{xcolor}         

\usepackage{amsmath, amsthm, amssymb, amsfonts}
\usepackage{bm}
\usepackage{mathtools}
\usepackage{tabularx}
\usepackage{thm-restate}
\usepackage{subfig}
\usepackage{color}
\usepackage{enumitem}
\usepackage{lipsum}

\usepackage[extdef=true]{delimset}
\usepackage{algorithm}
\usepackage{algorithmic}

\usepackage{mathtools}
\usepackage{thmtools}

\usepackage{microtype}
\usepackage{amssymb}
\usepackage{pifont}
\usepackage{nicefrac}
\usepackage{cancel}
\usepackage{array}
\usepackage{graphicx}
\usepackage{dsfont}
\usepackage{booktabs} 
\usepackage{amssymb}
\usepackage{color}
\usepackage{amsmath}
\usepackage{amsthm}
\usepackage{thmtools}
\usepackage{thm-restate}
\usepackage{algorithm}
\usepackage{algorithmic}
\usepackage{enumitem}
\usepackage{setspace}
\usepackage{thm-restate}
\usepackage{nameref}

\usepackage[margin=1in]{geometry}

\usepackage[square,numbers,sort&compress]{natbib}
\usepackage[colorlinks, linkcolor=blue, filecolor = blue, citecolor = blue, urlcolor = blue]{hyperref}
\usepackage[capitalize]{cleveref}

\allowdisplaybreaks

\newtheorem{lem}{Lemma}

\newtheorem{cor}{Corollary}
\newtheorem{defn}{Definition}

\newtheorem{rem}{Remark}

\newcommand{\Tau}{\scalebox{1.44}{$\tau$}}
\newcommand{\R}{{\mathbb R}}

\newcommand{\N}{{\mathbb N}}
\renewcommand{\P}{{\mathbf P}}

\newcommand{\E}{{\mathbf E}}
\newcommand{\I}{{\mathbf I}}
\newcommand{\1}{{\mathbf 1}}
\newcommand{\0}{{\mathbf 0}}
\newcommand{\eps}{{\epsilon}}
\newcommand{\cA}{{\mathcal A}}

\newcommand{\cB}{{\mathcal B}}

\newcommand{\cC}{{\mathcal C}}
\newcommand{\hcB}{{\hat \cB}}
\newcommand{\tcB}{{\tilde \cB}}
\newcommand{\cH}{{\mathcal H}}

\newcommand{\cS}{{\mathcal S}}

\newcommand{\cG}{{\mathcal G}}

\newcommand{\cO}{{\mathcal O}}
\newcommand{\cD}{{\mathcal D}}

\newcommand{\cR}{{\mathcal R}}

\newcommand{\hw}{{\hat \w}}

\renewcommand{\a}{{\mathbf a}}
\renewcommand{\b}{{\mathbf b}}
\newcommand{\e}{{\mathbf e}}

\newcommand{\p}{{\mathbf p}}

\newcommand{\w}{{\mathbf w}}
\newcommand{\x}{{\mathbf x}}
\newcommand{\y}{{\mathbf y}}
\newcommand{\z}{{\mathbf z}}
\newcommand{\sm}{\setminus}
\newcommand{\supp}{\text{Supp}}

\newcommand{\bign}[1]{\big(#1\big)}
\newcommand{\biggn}[1]{\bigg(#1\bigg)}

\newcommand{\biggsn}[1]{\bigg[#1\bigg]}

\def \dpdb{{\it DP-DB}}
\def \papertitle{DP-Dueling: Learning from Preference Feedback without Compromising User Privacy} 

\newcommand{\tO}{{\tilde O}}
\newcommand{\algfin}{DP-EBS-Elimination}
\newcommand{\alggen}{GOptimal-DPDB}

\newcommand{\ceil}[1]{\Big \lceil{#1} \Big\rceil}

\newcommand{\BinTr}{\mathsf{BinTree}}

\newcommand{\ha}[1]{
		\textcolor{blue}{\textbf{HA:} {#1}}
}

\newcommand{\as}[1]{
		\textcolor{red}{\textbf{AS:} {#1}}
}

\newcommand{\ed}{\ensuremath{(\eps,\delta)}}

\allowdisplaybreaks

\title{\papertitle}

\author{}
\author{
Aadirupa Saha%
\thanks{Apple. {\tt aadirupa@ttic.edu}.}
\and 
Hilal Asi%
\thanks{Apple. {\tt hasi@apple.com}.} 
}
\date{}

\begin{document}

\maketitle

\begin{abstract}
    We consider the well-studied dueling bandit problem, where a learner aims to identify near-optimal actions using pairwise comparisons, under the constraint of differential privacy.
    We consider a general class of utility-based preference matrices for large (potentially unbounded) decision spaces and give the first differentially private dueling bandit algorithm for active learning with user preferences. Our proposed algorithms are computationally efficient with near-optimal performance, both in terms of the private and non-private regret bound. 
    More precisely, we show that when the decision space is of finite size $K$, our proposed algorithm yields order optimal $O\bign{\sum_{i = 2}^K\log\frac{KT}{\Delta_i} + \frac{K}{\epsilon}}$ regret bound for pure $\epsilon$-DP, where $\Delta_i$ denotes the suboptimality gap of the $i$-th arm. We also present a matching lower bound analysis which proves the optimality of our algorithms. 
    Finally, we extend our results to any general decision space in $d$-dimensions with potentially infinite arms and design an $\epsilon$-DP  algorithm with regret $\tilde O \left(   \frac{d^6}{\kappa \epsilon  } + \frac{ d\sqrt{T }}{\kappa} \right)$, providing privacy for free when $T \gg d$. 
    
\end{abstract}


\section{Introduction}
\label{sec:intro}

Research has indicated that it is often more convenient, faster, and cost-effective to gather feedback in a relative manner rather than using absolute ratings \citep{musallam2004cognitive,kahneman1982psychology}. To illustrate, when assessing an individual's preference between two items, such as A and B, it is often easier for respondents to answer preference-oriented queries like ``Which item do you prefer, A or B?'' instead of requesting to rate items A and B on a scale ranging from 0 to 10. 
From the perspective of a system designer, leveraging this user preference data can significantly enhance system performance, especially when this data can be collected in a relative and online fashion. This applies to various real-world scenarios, including recommendation systems, crowd-sourcing platforms, training bots, multiplayer games, search engine optimization, online retail, and more.

In many practical situations, particularly when human preferences are gathered online, such as designing surveys, expert reviews, product selection, search engine optimization, recommender systems, multiplayer game rankings, and even broader reinforcement learning problems with complex reward structures, it's often easier to elicit preference feedback instead of relying on absolute ratings or rewards. 

Because of its broad utility and the simplicity of gathering data using relative feedback, learning from preferences has become highly popular in the machine learning community. It has been extensively studied over the past decade under the name ``Dueling-Bandits" (DB) in the literature. This framework is an extension of the traditional multi-armed bandit (MAB) setting, as described in \cite{Auer+02}. In the DB framework, the goal is to identify a set of 'good' options from a fixed decision space, which consists of a set of items. This is accomplished by collecting preference feedback through actively chosen pairs of items. The problem has been addressed in various works such as \citep{Yue+12,Ailon+14,Zoghi+14RUCB,Zoghi+14RCS,Zoghi+15,VDB,SDB}, along with its subsetwise generalization e.g. \emph{Battling Bandits} \citep{SG18,SGwin18,SGrank18,SGinst20,Ren+18} or \emph{Multi-dueling Bandits} \citep{Brost+16,sui18survey,Sui+17,GS21,SGrum20,SG19}.  Over the years the preference-based learning setup is extended from finite to large or potentially infinite decision spaces \cite{Yue+09,saha21}, stochastic to adversarial preferences \cite{Adv_DB,ADB} and contextual scenarios
\cite{CDB,bengs2022stochastic}, item unavailability \cite{SGV20,gaillard2023one}, non-stationary preferences \cite{SahaNDB,KBS22,BengsNDB}. 
In fact, not only in the learning theory literature, but learning with preference feedback has been studied extensively in many interdisciplinary areas including 
Reinforcement Learning (RL) \citep{wirth16,sui19,xu20,april,PacchianoDuelingRL,busa14,shao2024eliciting},
Robotics \citep{wirth17,biyik2022aprel}, 
Multiobjective optimization \citep{Yue+09,kumagai2017regret,SKM21,SKM22,blum2024dueling,SKM23},
Game theory \citep{bai20,bai+20,liu21,CDB,RDB}, 
Econometrics \citep{assort-mallows,assort-mnl,SGa24}, 
Healthcare \citep{tucker2020preference,li2021roial}.


Amidst the growing surge of interest in leveraging preference feedback for machine learning, it is noteworthy that there exists a notable gap in the literature concerning the privacy of user preferences within the feedback data. This gap becomes particularly significant in real-world systems reliant on human preferences for training artificial intelligence models. For instance, recommender systems exemplified by industry giants like Netflix and YouTube, which derive insights from user preferences, such as click behavior or viewing history. The safeguarding of sensitive user information in these contexts is extremely important, necessitating the implementation of robust privacy-preserving mechanisms.

Likewise, in the domain of healthcare systems, e.g. personalized diet and activity recommendations, the preservation of an individual user's specific choices and preferences is not merely a preference but a fundamental requirement. Unfortunately, none of the existing dueling bandits methods have addressed this crucial aspect of preserving user privacy, highlighting a significant gap in the current landscape of preference-based machine learning with human feedback.

Towards that end, there has been several recent attempts to design online algorithms ensuring differentially privacy~\cite{SmithTh13,AgarwalSi17,AsiFeKoTa23}. However, most of these papers are limited to classical settings such as online convex optimization~\cite{AgarwalSi17}, prediciton from expert advice~\cite{AsiFeKoTa23}, or bandits~\cite{SmithTh13}. In this work, we initiate the study of privacy-preserving algorithms for duelling bandits, motivated by recent applications in the field of robotics \cite{tucker2020preference,dorsacar,dorsastall}, reinforcement learning with human feedback (RLHF) \cite{rlhf,rlhff,srlhf}, preference based RL \cite{PacchianoDuelingRL,wirth17} and large language models \cite{instructgpt,chatgpt,openai2022chatgpt}.

\paragraph{Informal Problem Setting.} In an active learning framework, we address the task of constructing differentially private prediction models from user preferences. We consider a decision space $\mathcal{D} \subset \mathbb{R}^d$, where each item $\mathbf{x} \in \mathcal{D}$ is associated with a reward parameter $r(\mathbf{x}) = \mathbf{w}^\top \mathbf{x}$ for an unknown weight vector $\mathbf{w} \in \mathbb{R}^d$. The learner's objective is to find the item with the highest reward $\mathbf{x}^* := \arg\max_{\mathbf{x} \in \mathcal{D}} r(\mathbf{x})$ with online access to user preferences.

At each round $t$, the learner (algorithm) selects a pair of items $(\mathbf{a}_t, \mathbf{b}_t)$ from the decision space $\mathcal{D}$ and receives binary preference feedback $o_t \in \{0,1\}$ indicating the action that is more favorable to the user between $a_t$ and $b_t$. The probability of $\mathbf{a}_t$ winning over $\mathbf{b}_t$ is modeled as Bernoulli random variable $o_t \sim \text{Ber}\big(\sigma(r(\mathbf{a}_t)-r(\mathbf{b}_t))\big)$, where $\sigma: \mathbb{R} \mapsto [0,1]$ is the logistic link function defined as $\sigma(x) = (1+e^{-x})^{-1}$. Here, $\P(o_t=1)$ represents the probability of $\mathbf{a}_t$ winning the duel against $\mathbf{b}_t$.

\paragraph{Objective. } Assuming $\x^*:= \arg\max_{\x \in \cD}r(\x)$ to be the highest rewarding arm, the goal of the learner is to minimize the $T$-round regret:   
\begin{equation*}
    \cR_T(\cA) := \sum_{t=1}^T \bigg(2r(\x^*) - r(\a_t) - r(\b_t) \bigg),
\end{equation*}
under the constraint of  $(\epsilon,\delta)$-Differential Privacy (more formal discussion is given in \cref{sec:obj}).


\subsection{Our Contributions}
\label{sec:contribution}
\begin{itemize}[nosep,leftmargin=0pt]
	\item We initiate the study of differentially private dueling bandits (\dpdb), aiming to understand the funndamental tradeoffs between learning with preferences and preserving the privacy of users.  We formally define the problem in~\Cref{sec:problem}.  
	
	\item We begin our contributions by studying the \dpdb~problem in the setting of finite decision space in~\Cref{sec:fin}. We develop an elimination-based algorithm that plays the arms in a round-robin fashion, and uses the binary tree mechanism to privately estimate the reward of each arm. We show that this algorithm achieves regret 
 $$\cR_T(\cA) = O\left( \sum_{i = 2}^K \frac{\log (KT)}{\Delta(1,i)} + \frac{K}{\epsilon} \right).$$

        \item Additionally, we provide a matching lower bound for any $\epsilon$-DP algorithm in~\cref{sec:lb}. Combined with existing lower bounds for non-private algorithms, this demonstrates the optimality of our algorithms for the finite armed case.

	\item Finally, we study the setting of unbounded decision space in~\Cref{sec:gen} where arms are $d$-dimensional vectors. A simple elimination-based algorithm for this setting would incur regret that is exponential in $d$ or a bad dependence on $T$. Therefore, we develop an elimination-based algorithm using core-sets for approximating the rewards of all arms. Briefly, we will have multiple phases, where at each phase we construct a core-set of small size that allows us to estimate the reward of all other arms, eliminating arms with small reward.  
    We show that our algorithm obtains regret 
$\tilde O \left(   \frac{d^6}{\kappa \epsilon  } + \frac{ d\sqrt{T}}{\kappa} \right)$; for the important setting of $\epsilon=1$ and $T \gg d$, our algorithm achieves the non-private regret while simultaneously preserving privacy.

\end{itemize}

\subsection{Problem Setting and Notation}
\label{sec:problem}

\textbf{Notations. } Let $[n] := \{1,2, \ldots n\}$ for any $n \in \N$. 
Lower case bold letters denote vectors, upper-case bold letters denote matrices.
$\I_d$ denotes the $d \times d$ identity matrix for $d \in \N$. For any vector $\x \in \R^d$, $\|\x\|_2$ denotes the $\ell_2$ norm of $\x$.
We denote by
$\Delta_{n}:= \{\p \in [0,1]^n \mid \sum_{i = 1}^n p_i = 1, p_{i} \ge 0, \forall i \in [n]\}$ denotes the $n$-simplex and $\e_i$ denotes the $i$-th standard basis vector, $i \in [K]$. $\cB_d(r)$ denotes the $d$-dimensional ball of radius $r$ for any $r \in \R_+$. We use $\tO(\cdot)$ notation to hide logarithmic dependencies.

\paragraph{Problem Setting.} 
\label{subsec:prob}
We consider the problem of designing differentially private effective prediction models through user preferences in an active learning framework: 
consider decision space $\cD \subset \cB_d(1)$. Assume any item $\x \in \R^d$ in the decision space is associated with a reward/score parameter $r(\x) = \w^{*\top} \x$ for some unknown $\w^* \in \R^d$. At each round, the goal of the learner (algorithm designer) is to find the item with the highest rewarding arm $\x^*:= \arg\max_{\x \in \cD}r(\x)$. 

However, the learner only has access to user preferences in an active sequential manner. At each round $t$, the learner is allowed to actively choose a pair of items $(\a_t,\b_t)$, upon which it gets to see a binary preference feedback $o_t \in \{0,1\}$ on the item pair indicating the winner of the dueling pair. Naturally, the arms with higher rewards would have a higher probability of winning, so one way to model the preference feedback could be $o_t \sim \text{Ber}\bign{\sigma(r(\a_t)-r(\b_t))}$, where 
 $\sigma: \R \mapsto [0,1]$ is the logistic link function, i.e. $\sigma(x) = (1+e^{-x})^{-1}$. Thus $\P(o_t=1) = \sigma(r(\a_t)-r(\b_t))$ denotes the probability of $\a_t$ winning over $\b_t$.



%

\paragraph{Differentially Private Dueling Bandit (\dpdb). }
\label{sec:obj}
In this paper, our goal is to solve the above duelling bandits problem under the constraint of differential privacy~\cite{DworkMcNiSm06}. In our setting, each user's sensitive data is the preference $o_t$ that the user provides. Following standard definitions of differential privacy in the online setting~\cite{SmithTh13,AgarwalSi17}, 
we use the following definition for our problem:


\begin{defn}[($\epsilon, \delta)$-differentially private dueling bandit]
\label{def:dpdb}
  A dueling bandit algorithm $\cA$ is $(\epsilon, \delta)$-differentially
  private if for all sequences $o_{1:T}$ and $o'_{1:T}$ that
  differs in at most one time step, we have for any possible outcome 
  $S \subseteq \cO$ in the set of all possible outcomes $\cO$:
\begin{align}
\P & (a_t,b_t,\dots,a_T,b_T \in S \mid o_{1:T})
\leq 
e^\epsilon \P(a_t,b_t,\dots,a_T,b_T \in S \mid o'_{1:T})
 + \delta
\label{eq:dp-bandit}
\end{align}
 When $\delta = 0$, the algorithm is said to be \emph{$\epsilon$-DP}.
\end{defn}

Intuitively, this means that changing the preference outcome $o_t$ for any given arm-pair $(a_t,b_t)$, will not have much effect on the algorithm's prediction at time $t$ or later on in the subsequent rounds. \emph{If each $o_t$ is private information or a point associated with a single individual, then the definition above means that the presence or absence of that individual will not affect too much the output of the algorithm. Hence, the algorithm will not reveal any extra information about this individual leading to privacy protection.}
The privacy parameters $(\epsilon, \delta)$ determine the extent/degree to which an individual entry affects the output; lower values of $(\epsilon, \delta)$ imply higher levels of privacy, and in this work we mainly consider the strongest notion of pure $\epsilon$-DP where $\delta=0$.
 
\paragraph{Performance Measure: Regret Minimization Under $(\epsilon,\delta)$-DP}
As motivated in the introduction, we want to address the problem of user privacy in systems that train on choice data, e.g. Recommender systems and Personalized healthcare. 
This requires balancing the tradeoff between optimizing the performance of the prediction model, without sacrificing user privacy.

Towards this, a sensible measure of performance could be to design an $(\epsilon,\delta)$-DP Dueling Bandit algorithm $\cA$ which is designed to optimize the total cumulative regret of the algorithm $\cA$ in $T$ rounds, defined as:
\begin{align}
\label{eq:reg}
	\cR_T(\cA) := \max_{\x^* \in \cD}\sum_{t=1}^T \bigg(2r(\x^*) - r(\a_t) - r(\b_t) \bigg).
\end{align}
\emph{This implies that the algorithm aims to optimize the regret performance (i.e. converge to the optimum item $\x^*$ as quickly as possible), without violating the differential privacy of the users.} We will refer to this problem as Differentially Private Dueling Bandit (\dpdb) in the rest of the paper.

\section{Preliminaries: Some Useful Concepts} 
\label{sec:prelims}

 \begin{defn}[$\varepsilon$-net for any set $\cS$~\citep{matouvsek1989construction,vershynin2018high}]
 \label{def:net}	
 Given any set $\cS\subseteq \{x\in \mathbb{R}^d| \|x\|_2\leq 1\}$ we define an $\varepsilon$-net of $\cS$ as a discrete set $\cS^{(\varepsilon)} \subseteq \mathcal{S}$  such that for any $\x \in \mathcal{D}$, there is some $\x'\in \cS^{(\varepsilon)}$ with $\|\x'-\x\|_2 \leq \varepsilon$.
\end{defn}
Throughout the paper, we will use the notation $\cS^{(\varepsilon)}$ to denote an $\varepsilon$-net of the set $\cS$. We also use the following standard fact for the size of $\varepsilon$-net for the unit ball.
\begin{rem}
\label{rem:net}
It is well known that if $\cS$ is the unit ball, then it is well known that (see ~\cite[Cor. 4.2.13]{vershynin2018high}), one can always find such a discrete $\varepsilon$-net of $\cS$ of cardinality at most $(\frac{3}{\varepsilon})^{d}+d$. 
\end{rem}

Our algorithms also require the notion of an optimal design, which is used in order to find a representative set of arms that is called the core-set.
 \begin{defn}[G-Optimal Design] \citep{lattimore2020bandit}
 \label{defn:gopt}	
Consider any set $\cD \subseteq \R^d$, and let $\pi: \cD \mapsto [0,1]$ be any distribution on $\cD$ (i.e. $\sum_{\x \in \cD}\pi(\x) = 1$). Let us define:
\[
V(\pi) = \sum_{\x \in \cD}\pi(\x)\x\x^\top, ~~~ 
g(\pi) = \max_{\x \in \cD}\|\x\|^2_{V(\pi)^{-1}}.
\]
In the subfield of statistics called optimal experimental design, the distribution $\pi$ is called a ``design", and the problem of finding a design that minimizes $g$ is called the G-Optimal design problem.
\end{defn}
The set $\text{Supp}(\pi)$ denotes the support of the distribution $\pi$ and is sometimes called the \emph{``Core Set"}. The following theorem characterizes the size of the core set and the minimum value of the function $g(\cdot)$.

\paragraph{Kiefer–Wolfowitz Theorem.}\hspace{-5pt}\citep{lattimore2020bandit}
By Kiefer-Wolfowitz lemma, we know that for any compact set $\cD \subseteq \R^d$, such that $\text{span}(\cD) = \R^d$, there exists a minimizer $\pi^*$ of $g$ such that $g(\pi^*) = d$; further the size of the support of $\pi^*$ is $|\text{Supp}(\pi^*)|\leq \frac{d(d+1)}{2}$. Also, finding such G-Optimal designs are computationally efficient for large class of compact sets $\cD$.

\paragraph{The Binary Tree Mechanism.}
Our algorithms use the binary tree mechanism~\cite{DworkNaPiRo10,ChanShSo11} which is a key tool in differential privacy that allows to privately estimate the running sum of a sequence of $T$ numbers $a_1,\dots,a_T \in [0,1]$. This mechanism has the following guarantees.
\begin{lem}[\cite{DworkNaPiRo10}, Theorem 4.1]
\label{lemma:bt}
    Let $\epsilon \le 1$.  There is an $\epsilon$-DP algorithm ($\BinTr$) that takes a stream of numbers $a_1,a_2,\dots,a_T \in [0,1]$ and outputs $c_1,c_2,\dots,c_T$ such that for all $t \in [T]$ and any $\delta \in (0,1)$, with probability at least $1-\delta$,
    \begin{equation*}
        \Big| c_t - \sum_{i=1}^t a_i \Big| 
        \leq
        \frac{4  \log(1/\delta) \log^{2.5} T}{\epsilon} .
    \end{equation*}
\end{lem}

\section{Warm Up: Finite armed DP-DB}
\label{sec:fin}

In this section, we first address the \dpdb\, problem for finite decision space $\cD = [K]$. 
For simplicity, let us index the arms as $\{1,2,\ldots, K\}$, and suppose that $\x_i$ denotes the feature of the $i$-th arm. Then
\[
\text{Prob($i$ beats $j$)} = P(i,j) = \sigma\bign{r(\x_i)-r(\x_j)}.
\]

\noindent
Let us also define $\Delta(i,j) = P(i,j) - 1/2$. 
Note that the preference relation $\P$ has a total ordering: there is a total order $\succ$ on $\cA$, such that $a_{i} \succ a_{j}$ implies $\Delta_{i,j}>0$.  
Without loss of generality, let us assume $1 \succ 2 \succ \ldots K$.


%


\noindent
Before going to the algorithm description, we need to introduce the concept of \emph{`Effective Borda Score' (EBS)}. Given any subset of items $\cS \subseteq [K]$, the `Effective Borda Score (EBS)' of an item $i$ in the subset $\cS$ is defined as the average probability of an arm $i$ winning against any random item from set $\cS$:
\[
\cB_\cS(i):= \frac{1}{|\cS|}\sum_{j \in \cS}P(i,j)
\]
\textbf{Key Properties of EBS. } It is easy to note that for our underlying pairwise relation $P(i,j)$, $\cB_\cS(i)$ follows the same total ordering as that of $P$. Precisely assuming the ordering $1 \succ 2 \succ \ldots K$, as defined above, for any set $\cS$, it is easy to note that $\cB_\cS(i) \succ \cB_\cS(j)$ for any pair of items $i,j \in \cS$ such that $i \succ j$ in the original total ordering.   

\noindent
Further, given the above definition of $\cB_S$, the EBS of any item $i$ follows two distinct properties: given any subset $\cS \subseteq [K]$, suppose $i^\star_\cS$ and $i'_\cS$ respectively denoted the best and worst item of the set $\cS$ w.r.t. the total ordering of $\cS$. Then one can show that: 
\begin{align*}
& \text{Property:}(1)~~ \cB_\cS(i^\star_\cS) - \cB_\cS(i'_\cS) \geq \Delta(i^\star_\cS,i'_\cS),  
\\
& \text{Property:}(2)~~ \forall j \in \cS\sm\{i_\cS^\star\},
\cB_\cS(j) - \cB_\cS(i'_\cS) \leq \Delta(i^\star_\cS,i'_\cS)
\end{align*}
Intuitively this defines a concept of `preference-gaps' among the arms, which we would use to identify and eliminate the suboptimal arms quickly. The proof of the above claims is given in the Appendix, which follows by exploiting the underlying utility-based structure of the preference relation $P$. These properties play a crucial role in our algorithm design, as described in \cref{sec:alg_fin}, as well as in showing its optimal regret performance. 

\subsection{Algorithm: \algfin}
\label{sec:alg_fin}

\textbf{Algorithm Idea: } The idea of the algorithm is to play the arms in a round-robin fashion and eliminate the arms that perform poorly in terms of comparative performance until we are left with one arm, which is provably the best arm with high probability. We use the binary-tree mechanism to estimate the performance of each arm under privacy constraints.
Precisely, the algorithm runs $(K-1)$ in phases $\tau = 1,\ldots, (K-1)$. The key ideas behind the algorithm are explained step by step:

\paragraph{(1) Round-Robin Duel Selection on the Active Set. }
At every round $t$, inside phase $\tau$, we maintain an active set of surviving items $\cS_\tau$, initialized as $\cS_1 = [K]$. 
For the clarity of exposition, let us denote by $\tau(t)$ the phase count of time $t$. 
The algorithms select pairs $(a_t,b_t)$ at round $t$ as follows: its plays the first arm in a round-robin fashion from the action set $\cS_{\tau(t)}$, and the second arm is played at random from the active set. We will later see in the proof of \cref{thm:fin} that this idea of arm-selection ensures unbiased estimation of the EBS scores of each arm $i \in \cS_{\tau(t)}$.

\paragraph{(2) Maintaining EBS Estimates and UCBs.}
Now at each round $t$ in phase $\tau$, we keep the empirical EBS estimate of each of the surviving arms $i \in \cS_
\tau$, denoted by $\hcB_{t}(i):= \frac{w_t(i)}{n_t(i)}$,  where $w_t(i)$ and $n_t(i)$ respectively keep track of the total win count and play count of arm $i \in [K]$ starting from the beginning. 
Note that, at any time $t$ of phase $\tau$, 
$\hcB_t(i) = \hcB_{\cS_{\tau(t)}}(i)$ where $ \hcB_{\cS_{\tau(t)}}(i)$ is the empirical estimate of $\cB_{\cS_{\tau(t)}}(i)$.
%

\paragraph{(3) Ensuring Privacy.} As the above estimates of $w_t(i)$ (and consequently $\hcB_t(i)$) are not private, we use the binary tree mechanism to provide private estimates for $w_t(i)$ for each arm $i \in [K]$ using a different binary tree counter for each arm. Precisely we denote by $\tilde w_t(i)$ the private estimate of $w_t(i)$ (obtained from the binary tree mechanism) and the corresponding private EBS estimate $\tcB_t(i) = \frac{\tilde w_t(i)}{n_t(i)}$ for all $i \in [K]$. It is important to note here that each user can affect at most two of these counters and hence the privacy budget per counter can be $\eps/2$ rather than $\eps/K$. This ensures that the noise added by the binary tree mechanism at each iteration is roughly $\mathsf{poly}(\log(TK))/\eps$ with high probability. 
We also maintain the upper confidence (UCB) and lower-confidence (LCB) estimates of the EBS scores, respectively defined as: UCB$_t(i):= \tcB_t(i) + \sqrt{\frac{\nicefrac{\log (KT}{\delta)}}{n_t(i)}} + \frac{16  \log(\nicefrac{\log (K}{\delta)}) \log^{2.5} T}{n_t(i) \eps}$,
and 
LCB$_t(i):= \tcB_t(i) - \sqrt{\frac{\nicefrac{\log (KT}{\delta)}}{n_t(i)}} - \frac{16  \log(\nicefrac{\log (K}{\delta)}) \log^{2.5} T}{n_t(i) \eps}$.
Note, to account for the additional noise added to ensure privacy, we have to increase our confidence bounds UCB$_t(i)$ and LCB$_t(i)$ accordingly.

\vspace{-5pt}
\begin{algorithm}[H]
   \caption{\textbf{\algfin }}
   \label{alg:fin}
\begin{algorithmic}[l]
   \STATE {\bfseries Input:} Set of arms $[K]$, Time horizon $T$
   \STATE {\bfseries Init:} $\cS_1 = [K]$, $\tau = 1$ (phase counter), $\delta \in (0,1/2]$
   \STATE Set $n_1(i)=0$, $w_1(i)=0$ for all $i \in [K]$
   \STATE Initialize $K$ binary tree counters $C^1,\dots,C^K$ using the algorithm $\BinTr$ with privacy parameter $\epsilon/4$
   \FOR {$t = 1,\ldots,T$}
   \STATE Select the left arm $a_t \in \cS_\tau$ in a round robin fashion, such that $a_t = \arg\min_{i \in \cS_\tau}n_t(i)$
   \STATE Play $b_t \in \cS_\tau$ uniformly at random. 
   \STATE Receive preference $o_{t}(a_t,b_t) \sim \text{Ber}(P(a_t,b_t))$
   \IF{$o_t(a_t,b_t) = 1$ (arm $a_t$ wins)}
   \STATE $w_{t}(a_t) = w_{t-1}(a_t)+1$
   \STATE Add $+1$ to the counter $C^{a_t}$
   \ENDIF
   \STATE $n_{t}(a_t) = n_{t-1}(a_t)+1$
   \STATE Use the binary tree counter $C^{a_t}$ to privately estimate $w_t(a_t)$
   \STATE Let $\hat w_t(a_t)$ denote the output of $C^{a_t}$ 
   \STATE Update $\tcB_t(a_t):= \frac{\tilde w_{t}(a_t)}{n_{t}(a_t)} $ 
   \STATE For all $i \in \cS_\tau$ set:
   \\ UCB$_t(i):=\tcB_t(i) + \sqrt{\frac{\nicefrac{\log (KT}{\delta)}}{n_t(i)}} + \frac{16  \log(\nicefrac{\log (K}{\delta)}) \log^{2.5} T}{n_t(i) \eps}$
   \\ LCB$_t(i):=\tcB_t(i) - \sqrt{\frac{\nicefrac{\log (KT}{\delta)}}{n_t(i)}} -  \frac{16  \log(\nicefrac{\log (K}{\delta)}) \log^{2.5} T}{n_t(i) \eps}$   
    \FOR{all $i \in \cS_\tau$}
    \IF{there exists a $j \in \cS_\tau$ s.t. UCB$_t(i) < $ LCB$_t(j)$:}
   \STATE $\cS_{\tau+1} \leftarrow \cS_{\tau} \setminus \{i\}$, $\tau = \tau+1$.
    \FOR{$j \in \cS_{\tau+1}$}
        \STATE $n_t(j) \gets n_t(j) - \sum_{t' = 1}^t1\{ (a_{t'},b_{t'})= (j,i) \}$
        \STATE $w_t(j) \gets w_t(j) - \sum_{t' = 1}^t1\{ o_t = 1, (a_{t'},b_{t'})= (j,i) \}$
        \STATE Add $-1$ to the counter $C^{j}$
    \ENDFOR    
    \ENDIF
    \ENDFOR  
   \IF{ $\cS_\tau$ is singleton (i.e. $|\cS_\tau|==1$)}
   \STATE Assign $\hat i \leftarrow \cS_\tau$ and break.
   \ENDIF
   \ENDFOR
   \STATE Play $(\hat i,\hat i)$ for rest of the rounds $t+1, \ldots, T$. 
\end{algorithmic}
\end{algorithm}
\vspace{-5pt}

\paragraph{(4) Updating Active Sets with Action Elimination.} Finally, at each round $t$ in phase $\tau$, the algorithm performs a screening over all surviving arms $i \in \cS_{\tau(t)}$ in the active set, and prunes the arms with `sufficiently low' EBS scores: precisely if there exists arms $i,j \in \cS_{\tau(t)}$ such that UCB$_t(i) < $ LCB$_t(j)$, then arm $i$ is eliminated. 
\emph{One important thing to note here is that upon eliminating any arm $i$ from the active set $\cS_{\tau(t)}$, we also remove all the comparisons and statistics related to that arm from $w_t(j)$ and $n_t(j)$ for all surviving arms $j \in \cS_{\tau(t)}$.} This ensures that at any time $t$, $\hcB_t(i)$ is indeed an unbiased estimate of $\cB_{\cS_{\tau(t)}}$ (see \cref{lem:unbiasedest} for details). 

\noindent
The algorithm continues this process, for almost $K-1$ phases, until only one arm is left in $\cS_\tau$, which with high probability retains the best arm, owing to the `appropriate concentration properties' of the EBS scores (see \cref{lem:est_concs}). This in turn implies the regret bound of \cref{alg:fin} (stated in \cref{thm:fin}). The key idea is that any $i$-th suboptimal arm could be eliminated in $\tO\big(\frac{1}{\Delta(1,i)^2} + \frac{1}{\Delta(1,i)\epsilon}\big)$ many rounds, leading to a final regret bound of at most
$\tO\big(\sum_{i = 2}^K \frac{\log (KT)}{\Delta(1,i)} + \frac{K}{\epsilon}\big)$, while ensuring privacy. The formal description of the algorithm is given in \cref{alg:fin} 



\subsection{Regret Analysis: \algfin} \label{sec:reg_fin}

The following theorem summarizes the guarantees of~\cref{alg:fin}. 

\begin{restatable}[Regret Analysis of \algfin]{thm}{regalgfin}
   \label{thm:fin}  
 For any $T >K$, \algfin ($\cA$)\, is $\eps$-DP and its total regret is upper bounded by
   $$\cR_T(\cA) = \tO\left( \sum_{i = 2}^K \frac{\log (KT)}{\Delta(1,i)} + \frac{K}{\epsilon} \right).$$
 \end{restatable}

 \begin{proof}[Proof sketch of \cref{thm:fin}]
 The proof is developed based on several lemmas. 
 Starting with the privacy proof, we have the following lemma.
 \begin{lem}
 \label{lemma:priv-fin}
     Let $\eps \le 1$, $T \ge 1$ and $K\ge1$.
     \Cref{alg:fin} is $\eps$-differentially private.
 \end{lem}

 \begin{proof}
     The proof follows from the guarantees of the binary tree mechanism as the output of the algorithm is post-processing of the output of the binary tree mechanism. Note that we have $K$ instantiations of the binary tree mechanism in the algorithm $C^1,\dots,C^k$, thus it is sufficient to show that $C^1,\dots,C^k$ is $\eps$-DP. 

     We now prove  $C^1,\dots,C^k$ is $\eps$-DP. Note that each user can affect at most two counters. Moreover, the user can contribute only twice to each of these counters: once when adding $+1$ and the other when subtracting $-1$. As $C^i$ is $\eps/4$-DP, we get that each of these counters is $\eps/2$-DP by group privacy, and both counters are $\eps$-DP. The claim follows.
 \end{proof}

Now we proceed to prove the regret guarantees of the algorithm. We start with the following lemma.
 \begin{restatable}[]{lem}{unbiasedest}
 \label{lem:unbiasedest}
 Consider any fixed $n$, subset $\cS \subseteq [K]$ and any $i \in \cS$. Given any item $i \in \cS$, define a random variable $\hcB_n(i):= \frac{1}{n}\sum_{t = 1}^n \1(\text{i wins over a random item in } \cS)$. Then $\E[\hcB_n(i)] = \cB_{S}(i)$.
 \end{restatable}

 Note that~\cref{lem:unbiasedest} implies that given any phase $\tau$, $\E[\hcB_{t}(i)] = \cB_{\cS_\tau}(i)$ for all $ i \in \cS_\tau$. Moreover, the binary-tree mechanism adds zero-mean noise, hence we have the following corollary.

 \begin{cor}
 \label{cor:unbiasedest}
 Consider any time step $t$ at a given phase $\tau$. 
 Then $\tcB_{t}(i)$ is an unbiased estimate of $\cB_{\cS_{\tau}}(i)$, specifically $E[\tcB_{t}(i)] = \cB_{\cS_\tau}(i)$ for all $i \in \cS_\tau$. 
 \end{cor}

The next lemma argues that $\cB_{\cS_\tau}(i)$ lies inside the confidence interval $ [\text{LCB}_t(i),\text{UCB}_t(i)]$ with high probability at all time-steps.
 \begin{restatable}[]{lem}{estconcs}
 \label{lem:est_concs}
 For any choice of $\delta \in (0,1/2]$, we have that 
 \[
 \P\big( \exists t \in \text{Phase-}\tau, \exists i \in \cS_\tau, \cB_{\cS_\tau}(i) \notin [\text{LCB}_t(i),\text{UCB}_t(i)] \big) \leq \delta
 \]
 \end{restatable}

 The next lemma shows that if the confidence intervals are valid, then the best arm will always belong to the surviving set of arms $\cS_{\tau}$.
 \begin{restatable}[]{lem}{winnerstays}
   \label{lem:winnerstays}
  Assume that for all $t,\tau$ and $i \in \cS_\tau$, we have that $\cB_{\cS_\tau}(i) \in [\text{LCB}_t(i),\text{UCB}_t(i)]$.
  Then the best arm $1 \in \cS_{\tau}$ for all phases $\tau \in \{1,2, \ldots K-1\}$.
 \end{restatable}
 The next lemma shows that if the confidence intervals are valid, then any sub-optimal arm cannot be pulled too many times by the algorithm.
 \begin{restatable}[]{lem}{loserout}
   \label{lem:loserout}
   Assume that for all $t,\tau$ and $i \in \cS_\tau$, we have that $\cB_{\cS_\tau}(i) \in [\text{LCB}_t(i),\text{UCB}_t(i)]$.
   Then any suboptimal arm $i$ can be pulled at most $\big(\frac{8\log \nicefrac{(KT}{\delta)}}{\Delta(1,i)^2} + \frac{64\log \nicefrac{(KT}{\delta)}\log^{2.5} T}{\Delta(1,i)\epsilon}\big)$ times by the algorithm.
 \end{restatable}

 We defer the proof of the above lemmas to \cref{sec:proof-lemmas-fin}. 
 Given the results of the above lemmas, we are now ready to complete the proof of \cref{thm:fin}. First, we set $\delta = 1/T$:
 and note that the regret incurred by the learner for playing any item-$i$ ($i \neq 1$) is $r(x_1) - r(x_i) \leq 2\Delta(1,i)$, 
 since by definition $\Delta(1,i):= \sigma(r(x_1) - r(x_i)) - 1/2$. Therefore, the regret of the algorithm is upper bounded by:
\vspace{-10pt}
 \begin{align*}
 \sum_{i = 2}^K&\bigg(\frac{8\log \nicefrac{(KT}{\delta)}}{\Delta(1,i)^2} + \frac{64\log \nicefrac{(KT}{\delta)}\log^{2.5} T}{\Delta(1,i)\epsilon}\bigg) \Delta(1,i) 
  \\
  &\leq \sum_{i = 2}^K \bigg(\frac{16\log (KT)}{\Delta(1,i)} + \frac{128\log(KT)\log^{2.5} T}{\epsilon}\bigg).
 \end{align*}
 This concludes the regret bound of \cref{thm:fin} with the additional fact of \cref{lem:winnerstays} that with high probability the optimal arm $1$ is never eliminated and hence once the suboptimal arms $i = 2,\ldots K$ are eliminated, the algorithm assigns $\hat i = 1$ and pulls it for the remaining rounds, without incurring any additional regret. Finally $\epsilon$-DP property follows immediately from Lemma~\ref{lemma:priv-fin}.
 \end{proof}

 \begin{rem}[Optimality of \cref{thm:fin}]
 The regret bound of \algfin\, as described in \cref{thm:fin} is orderwise optimal as this meets both the standard non-private lower bound of $K$-dueling bandits $\Omega\Big(\sum_{i = 2}^K \frac{\log (KT)}{\Delta(1,i)}\Big)$ as well as the $\Omega(K/\eps)$
 privacy lower bound to come in~\Cref{sec:lb}.
 \end{rem}

 \begin{rem}[Other models of DP]
  We note that it is possible to build on our algorithm and extend it to other models of differential privacy such as the local model or the shuffle model. For example, instead of using the binary tree to estimate the aggregate feedback, each user can use randomized response to privatize their own feedback. This will result in larger noise and therefore the confidence intervals have to be changed accordingly. Moreover, the final error bounds for this algorithm will be (as expected in the local model) worse than the bounds we obtained in this section for the central model.
 \end{rem}

\section{Lower Bound }
\label{sec:lb}

In this section, we derive an information-theoretic lower bound for \dpdb\ problem for finite decision space $\cD = [K]$. Our derived lower bound in \cref{thm:LB_fin} corroborates the regret upper bound of our proposed algorithm \cref{alg:fin} (see \cref{thm:fin}), proving the optimality of our results. 
\begin{restatable}[Lower Bound]{thm}{regalgfin}
\label{thm:LB_fin}
Let $K$ and $\epsilon$ be such that $K/\epsilon \le T/2$.
For any $(\epsilon,\delta)$ differentially private algorithm $\cA$\, there exist a preference matrix $P$ such that the regret of $\cA$ on $P$ is lower bounded by
  $$\cR_T(\cA) = \Omega\left( \sum_{i = 2}^K \frac{\log T}{\Delta(1,i)} + \frac{K}{\epsilon} \right),$$
where $\Delta(1,i):= P(1,i) - 1/2$ simply defines the `preference gap' of the $i$-th arm against the best (rank-1) arm, item-$1$ in our case.  
\end{restatable}

We defer the proof of \cref{thm:LB_fin} to the \cref{app:lb}.



\vspace{-20pt}
\section{General Action Space}
\label{sec:gen}

We now proceed to address the \dpdb\, problem for more general decision spaces (with potentially large/infinite items). The main difficulty in this setting, unlike the finite decision space case discussed in \cref{sec:fin}, is that it is computationally infeasible to maintain pairwise estimates of every single distinct pairs. Hence, in this case, it is important to exploit the underlying low rank structure of the preference relation. Towards this, the main idea lies in deriving an estimate of $\w \in \R^d$ and maintaining confidence ellipsoids on the pairwise score-differences of the arms, defined as $s(\x,\y):= r(\x)-r(\y) = \w^{*\top}(\x-\y)$ for all $ \x,\y \in \cD$. We describe the key ideas of our algorithm in \cref{sec:alg_gen}. 
Further our performance analysis of \cref{alg:gen} shows it has near optimal regret guarantee under privacy constraints, as we proved in \cref{thm:fin}. 

\subsection{Algorithm: \alggen}
\label{sec:alg_gen}

Our main algorithm, \alggen\, is constructed by threading several key ideas carefully. Overall, same as \algfin\ (\cref{alg:fin}), this algorithm also runs in phases $\ell = 1, 2, \ldots$ and maintains an active decision space $\cD_\ell \subseteq \cD$. Note if $\cD$ is infinite, then we can initialize $\cD_1 \leftarrow \varepsilon$-net of $\cD$.  Inside any phase $\ell$, the algorithm first identifies a G-Optimal design (see \cref{defn:gopt}) of support size at most $d^2$ and only plays the arms in $\pi_\ell$. Once the items in the coresets are sufficiently explored, the suboptimal items are simply eliminated using the confidence interval estimation.

\noindent 
Our algorithm is a type of phased elimination based algorithm. 
As usual, at the end of a phase, arms that are likely to be suboptimal with a gap exceeding the current target are eliminated. In fact, this elimination is the only way the data collected in a phase is being used. In particular, the actions to be played during a phase are chosen based entirely on the data from previous phases: the data collected in the present phase do not influence which actions are played. This decoupling allows us to make use of the tighter confidence bounds available in the fixed design setting, as discussed below. The choice of policy within each phase uses the solution to an optimal design problem to minimize the number of required samples to eliminate arms that are far from optimal. This approach is inspired by the G-Optimal design algorithm of \cite{lattimore2020bandit} for the problem of linear bandits. 

\noindent
The design of our algorithm involves several key ideas and concepts, as listed below:

\paragraph{(1) Estimating $\w^*$ with Maximum Likelihood.} Note at any round $t$, if $\cup_{\tau = 1}^t\{(\x_\tau,\y_\tau,o_\tau)\}$ is the data we have seen so far, it is easy to formulate the data likelihood w.r.t. any parameter $\hw \in \R^d$: 
\begin{align*}
&{ L_t(\hw) 
= \log \biggsn{\Pi_{\tau = 1}^t \biggn{\sigma\bign{\hw^\top (\x_t - \y_t)}}^{o_t} \biggn{\sigma\bign{\hw^\top (\y_t - \x_t)}}^{1-o_t}}}
\\
& { = \sum_{\tau=1}^t \biggsn{o_\tau \log\biggn{ \frac{1}{1+e^{\hw^\top (\y_\tau-\x_\tau)}}}
+
(1-o_\tau) \log\biggn{ \frac{1}{1+e^{\hw^\top (\x_\tau-\y_\tau)}}}}}
\end{align*}
Taking derivative w.r.t. $\hw$ this implies that the maximum likelihood estimator $\hw$ would satisfy:
\begin{align*}
\sum_{\tau =1}^t\Big(o_\tau - \sigma\big( (\x_\tau-\y_\tau)^\top \hw  \big) \Big)(\x_\tau-\y_\tau) = \0.
\end{align*}

\paragraph{(2)  G-Optimal Design and Confidence Ellipsoid of the estimated scores $\hw_\ell^\top \x$ for any $\x \in \cD_\ell$.}
At the beginning of any phase $\ell$, we first identify the pairwise dueling space $\cD^2_
 \ell:= \{ \z_{\x,\y}:= \x-\y \in \R^d \mid \x, \y \in \cD \}$ and find a $G$-Optimal design $\pi_\ell \in \Delta_{\cD^2_\ell}$, such that \supp$(\pi_\ell) \leq d(d+1)/2$ (by Kiefer-Wolfowitz lemma, see \cref{defn:gopt}). Upon identifying the G-Optimal design in phase $\ell$, we play each dueling pair $(\x,\y)$ in the support of the G-Optimal design `enough times', precisely
         \[
         T_\ell^{\x,\y} = \Theta \left( \ceil{\frac{d^5\log(dT/\delta)}{\kappa \epsilon \xi_\ell } + \frac{d {\log(T/\delta)}}{\kappa^2 {\xi_\ell}^2}}\pi_{\ell}(\z_{\x,\y}) \right),
         \]
where $\kappa:= \inf_{\|\x-\y\|\le 2, \|\w^* - \hw\| \le 1} \Big[\sigma'\big( (\x-\y)^\top\hw \big)\Big]$ define a bound on the minimum slope of the sigmoid function. 


\noindent
Upon identifying the G-Optimal design $\pi_\ell \in \Delta_{\cD^2}$ of phase $\ell$ and exploring every pair $(x,y)$ by          $T_\ell^{\x,\y}$ times, we use the MLE estimation technique (described above) to estimate the unknown utility vector $\w_\ell$ at phase $\ell$. 
Now using \cref{lem:glm_ucb}, we can further show that with probability at least $(1-\frac{1}{TK})$,
\begin{align*}
|\x^\top(\hw_t - \w^*)| \le \xi_\ell:= 2^{-\ell}.
\end{align*}

\paragraph{(3) Action Elimination with Confidence Estimation.}
The above confidence estimation step for deriving the estimated scores $\hat s(\x):= \hw_\ell^\top\x$ of each item $\x \in \cD_\ell$ is crucial behind the correctness of our action elimination idea. This helps us to identify the arms whose optimistic estimates of the scores are below a certain threshold, relative to the other arms in $\cD_\ell$ and subsequently those poor performing arms are eliminated from phase $\ell$ as:
\[
\cD_{\ell+1} \leftarrow \{\x \in \cD_\ell \mid \min_{\y \in \cD_\ell}\hw^\top_\ell(\x-\y) + 2\xi_\ell > 0) \}.
\]
Note above implies for $\y \in \cD_\ell$, if there exists an arm $\x \in \cD_\ell$ such the the optimistic score estimate (UCB) of $\y$ falls below the pessimistic score estimate of arm $\x$, precisely $\hw_\ell^\top \y + \xi_\ell < \hw_\ell^\top \x - \xi_\ell$, then $\y$ is eliminated from $\cD_\ell$.
Recall we used the same elimination idea in \cref{alg:fin} as well. 

The algorithm thus proceeds in phases $\ell = 1, 2 \ldots$, and in case $\cD_\ell$ becomes singleton at any time $t$, it commits to that remaining item for the rest of the game $t+1,\ldots, T$. For ease of notation, we will use $\Tau_\ell$ to denote the time steps within phase $\ell$. 

\paragraph{(4) Adjustments to Ensure Privacy Guarantees.}
The only step which depends on the sensitive data $o_1, \dots, o_T$ is step $2$ which calculates the maximum likelihood estimator $\hat w$:
\begin{align*}
\sum_{\tau =1}^t o_\tau (\x_\tau-\y_\tau) = 
\sum_{\tau =1}^t \sigma\big( (\x_\tau-\y_\tau)^\top \hw_\ell  \big) (\x_\tau-\y_\tau) .
\end{align*}
One approach to privatize the MLE estimator is to add noise to each $o_t$, hence satisfying the stronger local DP. However, this noise will grow proportionally to $\sqrt{T}$ and the privacy error as well. Another approach is to use the binary tree mechanisms to estimate this $d$-dimensional aggregation problem. However, the noise will not be in the span of $\cC_\ell$, and the utility analysis does not work. To address these challenges, we chose to split the responses $o_1,\dots,o_T$ to groups based on $(x_t,y_t) \in C_\ell$, and privatize the aggregate sum for each group. As our coreset $C_\ell$ has at most $d^2$ arms, this will not incur too much additional noise. More precisely, for each $(x,y)\in \cC_\ell$, we compute the private aggregate response
\begin{equation*}
    \hat o_\ell(x,y) = \sum_{t \in \Tau_\ell} o_t 1\{(\a_t,\b_t) = (\x,\y) \} + \varepsilon_{xy}, \text{ where } \varepsilon_{xy} \sim \mathsf{Lap}(1/\eps)
\end{equation*}
then, we will find $\hw_\ell$ which satisfies:  
\begin{equation*}
        \sum_{(\x,\y) \in \cC_{\ell}} \hat o_\ell(\x,\y) (\x-\y) = 
        \sum_{t \in \Tau_\ell} \sigma\big( (\x_t-\y_t)^\top \hw_\ell  \big)(\x_t-\y_t) .
\end{equation*}
The complete algorithm and the pseudocode are given in \cref{alg:gen}. 

\vspace{-5pt}
\begin{algorithm}[H]
\caption{\alggen}
\label{alg:gen}
\begin{algorithmic}[l]
        \STATE \textbf{Input:} Decision space $\cD$, Time horizon $T$, Privacy parameter $\epsilon$, Exploration parameter $t_0$
        \STATE \textbf{Init:} $t_\ell = 1$. $\cD_1 \leftarrow \cD^{(1/T)}$ (i.e. the $1/T$-net of $\mathcal{D}$, see \cref{def:net}), $\delta$: Confidence parameter 
        \STATE Define $\kappa:= \inf_{\|\x-\y\|\le 2, \|\w^* - \hw\| \le 1} \Big[\sigma'\big( (\x-\y)^\top\hw \big)\Big]$ 
	\FOR{$\ell = 1,2,\ldots $} 
	\STATE Define the dueling set $\cD^2_
 \ell:= \{ \z_{\x,\y}:= \x-\y \in \R^d \mid \x, \y \in \cD \}$ 
	\STATE Find a $G$-Optimal design $\pi_\ell \in \Delta_{\cD^2_\ell}$, such that \supp$(\pi_\ell) \leq d(d+1)/2$ (by Kiefer-Wolfowitz lemma, see \cref{defn:gopt}) 
        \STATE Select $t_0$ pairs $\{(\a_t,\b_t)\}_{t \in [t_0]} \sim \pi_1$, and observe the corresponding preference feedback $\{o_t\}_{t \in [t_0]}$ 
        \STATE $\cC_\ell:= \{ (\x,\y) \in \cD_\ell \times \cD_\ell \mid \pi_\ell(\z_{\x,\y}) >0 \}$
        \STATE Let $\xi_\ell = 2^{-\ell}$. $\forall (\x,\y) \in \cC_\ell$, define 
         \begin{align*}
         T_\ell^{\x,\y}:=
         \ceil{\frac{16d^5\log(dT/\delta)}{\kappa \epsilon \xi_\ell } + \frac{64 d  {\log(T/\delta)}}{\kappa^2 {\xi_\ell}^2}}\pi_{\ell}(\z_{\x,\y})
         \end{align*}
 
        \FOR{$ \forall (\x,\y) \in \cC_{\ell}$} 
        \STATE Query the duel $(\a_t,\b_t) = (\x,\y)$ for $T_\ell^{\x,\y}$ times and receive the binary feedback $o_t$  
        \ENDFOR
        \STATE Define $\Tau_\ell = [t_\ell, t_\ell + t_0 + \sum_{(\x,\y)\in \cC_\ell}T_\ell^{\x,\y}]$. $t_\ell \leftarrow t_\ell + t_0 + \sum_{(\x,\y)\in \cC_\ell}T_\ell^{\x,\y}+1$
        \STATE Compute $\hat o_\ell(x,y) = \sum_{t \in \Tau_\ell} o_t 1\{(\a_t,\b_t) = (x,y) \} + \varepsilon_{xy}$, where $\varepsilon_{xy} \sim \mathsf{Lap}(1/\eps)$ 
        \STATE Compute $\hw_\ell$ that satisfies:
        \begin{align*}
        \sum_{(x,y)\in \cC_{\ell}}\hat o_\ell(x,y) (x-y)
        = \sum_{t \in \Tau_\ell} \sigma\big( (\a_t-\b_t)^\top \hw_\ell  \big)(\a_t-\b_t) 
        \end{align*}
        \STATE Define $V_\ell = \sum_{t \in \Tau_\ell}(\a_t-\b_t)(\a_t-\b_t)^\top$ 
        \STATE $\cD_{\ell+1}:= \{\x \in \cD_\ell \mid \min_{\y \in \cD_\ell}\hw^\top_\ell(\x-\y) + 2\xi_\ell > 0) \}$
        \STATE If $\cD_{\ell+1}$ is singleton, exit the for loop.
	\ENDFOR
	\STATE Let $\cD_\ell = \{\hat \x\}$. Play the duel $(\hat \x,\hat \x)$ for the remaining timesteps.		
\end{algorithmic}
\end{algorithm}
\vspace{-5pt}



\subsection{Regret Analysis: DP-GOptimal}
\label{sec:reg_gen}

\begin{restatable}[Regret of \cref{alg:gen}]{thm}{reggen}
  \label{thm:reg_gen}
  Let $\kappa := \inf_{\{\norm{\x} \in \cB_d(1),\,\,\norm{\w-\w^*} \le 1\}} \dot{\sigma}(\x^\top\theta)  > 0 $ and $\w^* \in \cB_d(1)$. 
  Consider any $d$-dimensional decision space $\cD \subseteq \cB_d(1)$. Then for the choice of 
  $t_0 = 2\Bigg( {C_1 d \sqrt d + C_2 d \sqrt{\log (1/\delta)}} \Bigg)^2 + { \frac{16d }{\kappa^4} \left(d^2 + \log\frac{T}{\delta}\right)}$, 
  and any finite horizon $T$, \cref{alg:gen} is $\epsilon$-DP with regret guarantee at most $O \left(   \frac{d^6\log(dT/\delta) \log({T}\kappa^2/d)}{\kappa \epsilon  } + \frac{ d\sqrt{T}   {\log(T/\delta)}}{\kappa} \right)$. 
\end{restatable}

\begin{proof}[Proof Sketch of \cref{thm:reg_gen}] 
%
%
First, note that the privacy of~\cref{alg:gen} follows immediately from the guarantees of the Laplace mechanism: indeed, each user $(a_t,b_t)$ affects only a single $\hat o_\ell(x,y)$ at a single phase with sensitivity $1$, hence the Laplace mechanism implies it is private.

For the regret guarantee, we will first prove the following crucial concentration result.
\begin{restatable}[]{lem}{gencon} 
\label{lem:glm_ucb}
Let $\delta \in (0,1)$.
Consider any phase $\ell = 1,2,\ldots \log T$. 
Recall we denote $V_\ell = \sum_{t \in \Tau_\ell} \z_t \z_t^\top$, where $\z_t = (\a_t-\b_t)$ for all $ t \in \Tau_\ell$. Assume that 
\begin{equation}
\lambda_{\min}(V_\ell) \ge  \frac{8}{\kappa^4} \left(d^2 + \log\frac{T}{\delta}\right). 
\label{eq:lmin1}
\end{equation}  
Then, with probability at least $1-3\delta$,  the maximum-likelihood estimator $\hw_\ell$ satisfies 
\begin{align} 
\label{eq:normality}
|\x^\top_1(\hat{\w}_\ell-\w^*)| 
& \le 8d^4 \frac{\log(d/\delta) \log(T)}{\kappa \epsilon} \max_{(\x,\y) \in \cC_{\ell} } \norm{\x-\y}_{V_\ell^{-1}}^2 +  \frac{8\gamma \sqrt{\log(T/\delta)}}{\kappa} \norm{\x_1}_{V_\ell^{-1}}.
\end{align}
for any $\x_1 \in \cD$.
\end{restatable}

The result of \cref{lem:glm_ucb} along with the duel-selection rule $(T_\ell^{\x,\y})$ of \cref{alg:gen} and phasewise initial exploration of $t_0$ pairs further implies:

\begin{restatable}[]{cor}{corglmucb}
\label{cor:glm_ucb}
For any phase $\ell = 1,2,\ldots \log T$, $\delta \in (0,1)$, and any $\x \in \cD$, with probability at least $(1-3\delta)$, we have that 
$
~|\x^\top(\hw_\ell - \w^*)| \le \xi_\ell = 2^{-\ell}
$. 
\end{restatable}






\begin{restatable}[]{lem}{epchlen}
  \label{lem:epchlen}
The length of each phase $\ell$ is bounded by $n_\ell \leq t_0 + {\frac{16d^5\log(dT/\delta)}{\kappa \epsilon \xi_\ell } + \frac{64  d {\log(T/\delta)}}{\kappa^2 {\xi_\ell}^2}} + d^2$, for all phase $\ell = 1,2,\ldots \log T$. 
\end{restatable}


%
%
\vspace{-10pt}
\begin{restatable}[]{lem}{beststays}
  \label{lem:beststays}
  Let $B = \Big|\cB_{d}(1)^{(1/T)}\Big|$ be the cardinality of $1/T$-net of $\cB_d(1)$, and $\x^* = \arg\max_{\x \in \cD} \w^{*\top} \x$. Then
  $
  Pr\biggn{\exists \ell \mid \x^* \notin \cA_\ell} \leq 3\delta B \ell$.
\end{restatable}
\vspace{-10pt}
\begin{restatable}[]{lem}{badout}
  \label{lem:badout}
  Let $B = \Big|\cB_{d}(1)^{(1/T)}\Big|$ be the cardinality of $1/T$-net of $\cB_d(1)$, and $\x^* = \arg\max_{\x \in \cD} \w^{*\top} \x$. Moreover, given any item $\x \in \cD$, let $\Delta_{\x}:= \w^{*\top} (\x^* - \x)$ denotes the suboptimality gap of item $\x$. Then if $\ell_x:= \min\{\ell \mid 2 \xi_\ell < \Delta_\x\}$ denotes the first phase when the targeted suboptimality falls below $\Delta_\x$, then
$Pr\biggn{\x \in \cA_{\ell_\x}} \leq 3\delta B \ell$.
\end{restatable}

\noindent
Using these lemmas, to bound the final regret, first note that with probability at least $1-3\delta B$, for all $\ell$ and all $\x \in \mathcal{\cD}_\ell$, we have  $|\langle \x,\hw_\ell \rangle-\langle \x,{\w^*} \rangle|\leq \xi_\ell$, as detailed in the algorithm description above. 
Conditioned on this event, an action with gap $\Delta_\x$ is eliminated when, or before, $\xi_\ell < \Delta_\x/2$. Hence, all actions in items in phase $\ell$ have gap at most $4\xi_\ell$. Moreover noting that, by definition, $T \ge n_\ell \ge d \log(T/\delta)/\kappa^2 \xi_\ell^2$, we have that $\xi_\ell \ge  \sqrt{d \log(T/\delta)}/\kappa\sqrt{T}$ and hence the number of phases $\ell \le \log({T}\kappa^2/d)/2$.
The expected regret bound follows by summing $4\xi_\ell n_\ell$ for all $\ell$ phases, as follows:

\begin{align*}
{R}_T(\cA) 
    &\leq 3 \delta B \ell T + \sum_{\ell=1}^{\log({T}\kappa^2/d)/2}4n_\ell \xi_{\ell} \\
    & \le  3 \delta B \ell T + 4 \sum_{\ell=1}^{\log({T}\kappa^2/d)/2}  \biggn{ t_0\xi_\ell + \frac{16 d^5\log(dT/\delta)}{\kappa \epsilon } + \frac{64  d {\log(T/\delta)}}{\kappa^2 {\xi_\ell}}} + d^2 \\
    & \le  3 \delta B \ell T + O \left(  t_0 \log(T\kappa/d) +  \frac{d^5\log(dT/\delta) \log({T}\kappa^2/d)}{\kappa \epsilon  } + \frac{ \sqrt{T d   {\log(T/\delta)}}}{\kappa} \right).
\end{align*}

As the size of $1/T$-net of the unit ball is upper bounded by $B = (3T)^d + d$ (Remark~\ref{rem:net}),
we set $\delta = 1/ B T \ell$ and  complete the proof.
The complete proof, along with that of all the lemmas, is provided in \cref{app:gen}. 
\end{proof}



\begin{rem}[On the optimality of \cref{thm:reg_gen}]
The regret in Theorem~\ref{thm:reg_gen} is optimal up to $\log T$ factor 
a lower bound of $O(d\sqrt{T})$ is proven in \cite{rusmevichientong2010linearly,Yadkori11,dani08} for the non-private case, while a lower bound of $\frac{d \log T}{\epsilon}$ is shown in \cite{shariff2018differentially} for the private case, which shows near optimality in our privacy guarantee (up to dimension dependence). 
Moreover note if the privacy parameter $\epsilon$ is $O(1)$ or small, often the dominating term in the regret of \cref{thm:reg_gen}  is $\tilde O(d\sqrt{T})$ which implies we get privacy almost for free in the low-dimensional setting.
\end{rem}

\section{Conclusion}
\label{sec:concl}

In summary, our paper tackles the dueling bandit problem under privacy constraints. The primary focus is on utility-based preference matrices tailored for vast or potentially infinite decision spaces. The paper introduces new differentially private dueling bandit algorithms for active learning, striking a balance between computational efficiency and nearly optimal performance in both private and non-private regret bounds. In particular, for finite decision spaces of size $K$, our algorithm delivers an order-optimal regret bound of $O(\sum_{i=2}^K \log(KT)/\Delta_i + K/\epsilon)$ while maintaining $\epsilon$-DP, as substantiated by matching lower bound analysis. Additionally, we also extended our algorithms to encompass general $d$-dimensional decision spaces with possibly infinite arms, presenting an $\epsilon$-DP regret algorithm with  regret bound of $\tilde O \left(   \frac{d^6}{\kappa \epsilon  } + \frac{ d\sqrt{T}}{\kappa} \right)$.

\paragraph{Future Directions. } This work leads a primary foundation of learning with preference feedback in a privacy-preserving way. The literature of preference learning being vast \cite{Busa21survey}, there lies several open questions including understanding the problem complexities for dynamic and adversarial environments, incorporating user-level personalization, understanding the regret-vs-privacy tradeoff beyond utility-based preferences, amongst many.

\newpage

\bibliographystyle{plainnat}
\bibliography{refs_db,db_real,refs_dp}

\newpage

\appendix
\onecolumn{
\section*{\centering\Large{Supplementary: \papertitle}}
\vspace*{1cm}


\section{Appendix for \cref{sec:fin}}
\label{app:fin}


\subsection{Proof of EBS Properties: }
\begin{align*}
& \text{Property:}(1)~~ \cB_\cS(i^\star_\cS) - \cB_\cS(i'_\cS) \geq \Delta(i^\star_\cS,i'_\cS),  
\\
& \text{Property:}(2)~~ \forall j \in \cS\sm\{i_\cS^\star\},
\cB_\cS(j) - \cB_\cS(i'_\cS) \leq 2\Delta(i^\star_\cS,i'_\cS)
\end{align*}

\begin{proof}[Proof of the Properties of EBS.]
We start by recalling that: Given any subset of items $\cS \subseteq [K]$, the `Effective Borda Score (EBS)' of an item $i$ in the subset $\cS$ is defined as the average probability of an arm $i$ winning against any random item from set $\cS$:
\[
\cB_\cS(i):= \frac{1}{|S|}\sum_{j \in S}P(i,j).
\]
Further $i^\star_\cS$ and $i'_\cS$ respectively denotes the best and worst item of the set $\cS$ w.r.t. the total ordering of $\cS$.

By definition $P(i,j) = \sigma\bign{r(\x_i)-r(\x_j)}, ~\forall i,j \in [K]$, and we denote by $P(i,j) = \Delta(i,j) - 1/2$. Owning to the above structural property of the preference relations, we note that the resulting preference relation $\P$ respects the following two important properties:

\emph{Strong Stochastic Transitivity $(\mathrm{SST})$}: For any triplet $(i,j,k)$ such that $\Delta_{i,j}\geq 0$ and $\Delta_{j,k}\geq 0$, the inequality $\Delta_{i,k} \ge \max \left\{ \Delta_{i,j} , \Delta_{j,k} \right\} $ holds good.

\emph{Stochastic Triangle Inequality $(\mathrm{STI})$}: For any triplet $(i,j,k)$ such that $\Delta_{i,j}\geq 0$ and $\Delta_{j,k}\geq 0$, it holds that $\Delta_{i,k} \le \Delta_{i,j} + \Delta_{j,k}$.

The two EBS properties are easy to prove now using the properties of SST and STI, as shown below.

\textbf{\textbullet~ Proof of EBS Property-1.} Consider any subset $\cS \subseteq [K]$, as defined earlier $i^\star_\cS$ and $i'_\cS$ is the best and worst item in $\cS$.

Thus for any other arm $i \in \cS$, note $i^\star_\cS \succ i \succ i'_\cS$. Recall by definition
\[
\cB_\cS(i):= \frac{1}{|S|}\sum_{j \in S}P(i,j)
\]
Thus we get:
\begin{align*}
    \cB_\cS(i^\star_\cS) & - \cB_\cS(i'_\cS) = \frac{1}{|\cS|}\sum_{j \in \cS}\bign{P(i^\star_\cS,j) - P(i'_\cS,j)}
    \\
    & = \frac{1}{|\cS|}\sum_{j \in \cS}\bign{\Delta(i^\star_\cS,j) - \Delta(i'_\cS,j)} \geq \frac{1}{|\cS|}\sum_{j \in \cS}\bign{\Delta(i^\star_\cS,i'_\cS)} = \Delta(i^\star_\cS,i'_\cS),
\end{align*}
where the last inequality follows from the STI property which gives $\Delta(i^\star_\cS,j) + \Delta(j,i'_\cS) \geq \Delta(i^\star_\cS,i'_\cS)$ or equivalently 
$\Delta(i^\star_\cS,j) - \Delta(i'_\cS,j) \geq \Delta(i^\star_\cS,i'_\cS)$ (since for any pair $(i,j) \in [K] \times [K]$, $\Delta(i,j) = \Delta(j,i)$). This proves the first property. 

\textbf{\textbullet~ Proof of EBS Property-2.} To see the second property, first note that if $j = i'_\cS$ or $j = i^\star_\cS$, the inequality holds trivially. For any $j \in \cS\sm\{i'_\cS,i^\star_\cS\}$ note that:
\begin{align*}
    \cB_\cS(j) & - \cB_\cS(i'_\cS) = \frac{1}{|\cS|}\sum_{i \in \cS}\bign{P(j,i) - P(i'_\cS,i)}
    \\
    & = \frac{1}{|\cS|}\sum_{i \in \cS}\bign{\Delta(j,i) - \Delta(i'_\cS,i)} 
    = \frac{1}{|\cS|}\sum_{i \in \cS}\bign{\Delta(j,i) + \Delta(i,i'_\cS)},
    \\
    & = \frac{1}{|\cS|}\biggsn{\sum_{i \in \cS \mid \Delta(j,i)< 0}\bign{\Delta(j,i) + \Delta(i,i'_\cS)} + \sum_{i \in \cS \mid \Delta(j,i)\geq 0}\bign{\Delta(j,i) + \Delta(i,i'_\cS)}  }
    \\
    & \leq \frac{1}{|\cS|}\biggsn{\sum_{i \in \cS \mid \Delta(j,i)< 0}{\Delta(i^\star_\cS,i'_\cS)} + \sum_{i \in \cS \mid \Delta(j,i)\geq 0}{2\Delta(i^\star_\cS,i'_\cS)}  } \leq \Delta(i^\star_\cS,i'_\cS),
\end{align*}
where the last inequality follows from the SST property which implies for any $i$ s.t. $i \succ j$ (i.e. $\Delta(j,i) \geq 0$), $\Delta(j,i) \leq \Delta(i^\star_\cS,i) \leq \Delta(i^\star_\cS,i'_\cS)$. This concludes the second part.
\end{proof}

\subsection{Proof of Lemmas for \cref{thm:fin}}
\label{sec:proof-lemmas-fin}
\unbiasedest*

\begin{proof}
    The proof is easy to follow from the observation that:
    \begin{align*}
        E&[\hcB_n(i)]= \frac{1}{n}\sum_{t = 1}^n E[\1(\text{i beats a random item in } \cS)]
        = \frac{1}{n}\sum_{j \in S}\P\bign{\text{i beats j, j is the random opponent})}
        \\
        & = \frac{1}{n}\sum_{j \in S}\P\bign{\text{i beats j} \mid \text{j is the opponent}}P(\text{j is the random opponent}) = \frac{1}{n}\sum_{j \in S}P(i,j)\frac{1}{|\cS|} = \cB_{\cS}(i).
    \end{align*}
\end{proof}
%


\estconcs*

\begin{proof}
Recall given any phase $\tau$, we defined: 
$\hcB_{t}(i):= \frac{w_t(i)}{n_t(i)}$,
$\tcB_{t}(i):= \frac{\tilde w_t(i)}{n_t(i)}$, respectively as the unbiased non-private and private estimate of the EBS score of item $i$ at time $t$. Moreover using \cref{cor:unbiasedest}, $\E[\tcB_{t}(i)] = \cB_{\cS_\tau}(i)$. Note given any phase $\tau$, note that $w_t(i)$ simply denotes the number of times $i$ won against a random opponent in $\cS_\tau$. Also since the left arm $a_t$ is pulled in a round robin fashion from $S_{\tau}$, at any time $t$, $n_t(i)$ is deterministic and denotes the number of times $i$ was played as a left-arm and made to compete against a random opponent in $S_\tau$. Thus given any time $t$ in phase $\tau$, $n_t(i) = t \mod |\cS_\tau|$. 
%
Further, recall that we defined:
\begin{align*}
UCB_t(i):=\tcB_t(i) + \sqrt{\frac{\nicefrac{\log (KT}{\delta)}}{n_t(i)}} + \frac{16  \log(\nicefrac{\log (K}{\delta)}) \log^{2.5} T}{n_t(i) \eps}
   \\ LCB_t(i):=\tcB_t(i) - \sqrt{\frac{\nicefrac{\log (KT}{\delta)}}{n_t(i)}} -  \frac{16  \log(\nicefrac{\log (K}{\delta)}) \log^{2.5} T}{n_t(i) \eps}   
\end{align*}


Now denote $\hat w_t(i) = w_t(i)  + \zeta_t(i) $ where $\zeta_t(i)$ is the noise added by the counter $C^i$ at time $t$. 
Now consider any item $i \in \cS_\tau$ and round $t \in [T]$ such that $n_{i}(t) > 0$ (lemma is trivially true otherwise). Applying Hoeffding's Inequality over the non-private estimate $\cB_t(i)$ we get: 
\begin{align*}
	Pr \Bigg(|\cB_{\cS_\tau}(i)-\hcB_t(i)| > \sqrt{\frac{\ln (Kt/\delta)}{n_{t}(i)}} \Bigg) 
	 \leq 2e^{-2 n_{i}(t)\frac{\ln (Kt/\delta)}{n_{i}(t)}} = \frac{2\delta^2}{K^2t^2} \leq \frac{\delta}{K^2t^2},
	\end{align*}
 where the last inequality holds since $\delta < 1/2$. 
Taking union bound over all $i \in \cS_\tau$ and all $t \in $ Phase-$\tau$, we get:
   \begin{align*}
	Pr & \Bigg(\exists i \in \cS_\tau, t \in \text{Phase-}\tau \text{ s.t. } |\cB_{\cS_\tau}(i)-\hcB_t(i)| > \sqrt{\frac{\ln (Kt/\delta)}{n_{t}(i)}} \Bigg) 
        \\
	& \leq \sum_{t = 1}^T \sum_{i=1}^{K} \frac{\delta}{K^2t^2} \leq \sum_{t = 1}^\infty \frac{\delta}{Kt^2} \leq \frac{\delta\pi^2}{6K} \leq \delta,
   \end{align*}
where in the second last inequality we used $\sum_{t = 1}^\infty \frac{1}{t^2} < \frac{\pi^2}{6}$ and the last inequality uses $K \geq 2$. Moreover, the guarantees of the binary tree mechanism (\cref{lemma:bt}) imply that with probability $1-\delta$ we have for all $i$
\begin{equation*}
        \frac{1}{n_t(i)}\Big| \tilde w_t(i) - w_t(i) \Big| = \Big| \tcB_t(i) - \hcB_t(i) \Big|
        \leq
        \frac{16 \log(K/\delta) \log^{2.5} T}{\eps n_t(i)} .
\end{equation*}
Combining the above two concentration results with triangle inequality yields the desired claim.
\end{proof}


\winnerstays*

\begin{proof}
Recall from \cref{alg:fin}, the criterion of arm-$1$ to be eliminated at any time $t$ is only if 

\begin{align}
\label{eq:win1}
    \exists j \in \cS_\tau \text{ s.t.  UCB}_t(1) <  \text{LCB}_t(j)
\end{align}

Now according to the statement of the \cref{lem:winnerstays} assume the good event $\cG:= \forall t \in \text{Phase-}\tau, \forall i \in \cS_\tau, \cB_{\cS_\tau}(i) \notin [\text{LCB}_t(i),\text{UCB}_t(i)]$. Assuming $\cG$ holds good, note that at any time $t$ of phase $\tau$:

\begin{align*}
\text{LCB}_t(i) \leq \cS_\tau(i) \overset{(a)}{<} \cS_\tau(1) \leq \text{UCB}_t(1),     
\end{align*}
where the inequality $(a)$ holds since by definition $\cB_\cS(i) < \cB_\cS(1)$. This is a contradiction to condition \eqref{eq:win1} and the claim follows.
\end{proof}


\loserout*

\begin{proof}
We start by recalling that we define $\Delta(1,i):= P(1,i)-1/2$. Further since the concentration of \cref{lem:est_concs} holds good, \cref{lem:winnerstays} ensures $1 \in \cS_\tau$ for all phase $\tau$.

Now consider any phase $\tau$. For simplicity, we denote $\cS_\tau$ by $\cS$. Let us consider the worst surviving item-$i = i'_\cS \in \cS$ and suppose there exists a round $t$ such that $n_i(t) \geq \big(\frac{8\log \nicefrac{(KT}{\delta)}}{\Delta(1,i)^2} + \frac{64\log \nicefrac{(KT}{\delta)}\log^{2.5} T}{\Delta(1,i)\epsilon}\big)$. But this would imply:

\begin{align*}
    \text{UCB}_t(i) 
    & = \tcB_t(i) + \sqrt{\frac{\log \nicefrac{(KT}{\delta)}}{n_t(i)}} = \text{LCB}_t(i) + 2\sqrt{\frac{\nicefrac{\log (KT}{\delta)}}{n_t(i)}} + \frac{32  \log(\nicefrac{\log (K}{\delta)}) \log^{2.5} T}{n_t(i) \eps} 
    \\
    & < \cB_\cS(i) + \Delta(1,i). 
\end{align*}

However, by the first property of EBS, $\cB_\cS(1) - \cB_\cS(i) \geq \Delta(1,i)$, which implies $\text{UCB}_t(i) < \cB_\cS(i) + \Delta(1,i) \leq \cB_\cS(1) \leq \text{UCB}_t(1)$ and hence arm-$i$ satisfies the elimination criterion, proving the claim.
\end{proof}

\section{Appendix for \cref{sec:lb}}
\label{app:lb}

In this section, we prove our lower bound for the finite case in~\Cref{thm:LB_fin}. To this end, for $i \in [K]$, 
we define the random variable $N_t(i)$ to be the number of times the algorithm has played $i$ up to time $t$; note that this random variable has randomness from the preference matrix $P$ and the randomness of the algorithm $\cA$. 

Our result builds on the following key lemma, which states that any \ed-DP algorithm must have $\E[N_T(i)] \ge 1/\eps$ for all $i \in [K]$, where the randomness is over both $P$ and $\cA$.
\begin{lem}
\label{lemma:arm-lb}
    Let $K$ and $\eps \le 1$ be such that $K/\eps \le T/10$.
    Let $\cA$ be an \ed-DP algorithm. If $\cR_T(\cA;P) \le T/2$ for every preference matrix $P \in \R^{K \times K}$, then for all $i \in [K]$
    \begin{equation*}
        \E_{P,\cA}[N_T(i)] \ge 1/\eps.
    \end{equation*}
\end{lem}
Before proving~\Cref{lemma:arm-lb}, we turn to prove our main lower bound.
Let $P$ be a preference matrix with constant gap: we define 
\begin{equation*}
    P(i,j) = 
    \begin{cases}
        1/2 & \text{ if } i=j \\
        3/4 & \text{ otherwise if } i=1 \\
        1/4 & \text{ otherwise if } j=1 \\
        1/2 & \text{ otherwise } 
    \end{cases}
\end{equation*}
Note that $i^\star=1$ is the optimal arm for $P$, and it has constant gap $\Delta(1,j) = 1/4$ for any $j \neq 1$. Now note that the regret of $\cA$ with respect to $P$ is 
\begin{align*}
    \cR_T(\cA;P) 
    & = \E_{P,\cA} [\sum_{t=1}^T (1\{a_t \neq 1\} + 1\{b_t \neq 1\})/4] \\ 
    & = \frac{1}{4} \cdot \E_{P,\cA} [\sum_{i \neq 1} N_T(i) ] \\
    & \ge \Omega(K/\eps),
\end{align*}
where the last step follows from~\Cref{lemma:arm-lb}.

Having proved our main lower bound, we now go back to prove~\Cref{lemma:arm-lb}.
\begin{proof}[Proof of~\Cref{lemma:arm-lb}]

    Let $P$ be a preference matrix and assume without loss of generality that $i^\star = 1$ is the optimal arm. As $\cR_T(\cA;P) \le T/8$, this shows that $\E_{P,\cA}[N_T(1)] \ge T/2 \ge  1/\eps$. Now we prove the claim for $i \neq i^\star$. To this end, fix $i_0 \in [K]$ and consider a preference matrix $P'$ where $i_0$ is the optimal arm: that is, 
    \begin{equation*}
        P'(i,j) = 
        \begin{cases}
            1/2 & \text{ if } i=j  \\ 
            3/4 & \text{ if } i  = i_0 \\ 
            1/4 & \text{ if } j = i_0 \\ 
            P(i,j) & \text{ otherwise } 
        \end{cases}
    \end{equation*}
    Note that $i_0$ is the optimal arm in $P'$, and that $P$ and $P'$ are identical for arms in $[K] \setminus \{i_0\}$. As $i_0$ is the optimal arm for $P'$ where other arms has constant gap, we must have that $\E_{P',\cA}[N_T(i_0)] \ge T/2 \ge  10/\eps$. Now our goal is to use properties of differential privacy to argue that the algorithm has behave similarly for $P$ and pull the arm $i_0$ for $1/\eps$ times.

    To this end, let $T_0$ be the random variable such that $N_{T_0}(i_0) = 1/\eps$ if such $t$ exists (otherwise $T_0 = \infty$). Note that $Pr_{P',\cA}(T_0 = \infty) = Pr_{P',\cA}(N_T(i_0) < 1/\eps) \le 1/4$ as otherwise we would have 
    $\E_{P',\cA}[N_T(i_0)] \le 3/4 \cdot T + 1/4 \cdot 1/\eps < 9T/10$.

    Now we prove that 
    \begin{equation*}
     Pr_{P,\cA}(N_{T}(i_0) \ge 1/\eps) \ge 1/1000.  
    \end{equation*}
    This will prove the claim as $\E_{P,\cA}[N_T(i_0)] \ge Pr_{P,\cA}(N_T(i_0) \ge 1/\eps) \cdot 1/\eps \ge \Omega(1/\eps) $. Now we consider running the algorithm $\cA$ with preference matrices $P$ and $P'$. Let $T_0$ and $N_T(i_0)$ be the corresponding parameters for $P$ and similarly $T'_0$ and $N'_T(i_0)$ for $P'$. Let $\{(a'_t,b'_t,o'_t)\}_{t=1}^{T'_0}$ be the interaction of $\cA$ with $P'$ up to time $T'_0$. Note that there are at most $1/\eps$ time-steps $t \le T_0'$ such that $a'_t,b'_t \in \{i_0\}$. 
    Thus, we construct a new sequence of rewards in which we change the value of $o_t$ for $t$ such that $a'_t,b'_t \in \{i_0\}$ such that for these $t$'s we sample $o_t \sim P$ instead of $P'$. We let $o_1,\dots,o_t$ denote this sequence: note that $o_t \sim P$ and is different than $o'$ in at most $1/\eps$ locations. Finally, as we have changed the reward in at most $1/\eps$ times-steps, this implies that the outputs of the algorithm cannot change by much by group privacy, hence we have that 
    \begin{equation*}
     Pr_{P,\cA}( N_{T'_0}(i_0) \ge 1/\eps) \ge   e^{-1} Pr_{P',\cA}( N_{T'_0}(i_0) \ge 1/\eps) \ge 3/4e.
    \end{equation*}
    Noting that $Pr_{P',\cA}(T_0 = \infty) \le 1/4$, we have that $T \ge T_0$ in this case, therefore
     \begin{align*}
     Pr_{P,\cA}( N_{T}(i_0) 
     & \ge 1/\eps) \ge   Pr_{P,\cA}( N_{T'_0}(i_0) \ge 1/\eps, T'_0 \le T) \\
     & \ge Pr_{P,\cA}( N_{T'_0}(i_0) - Pr_{P',\cA}(T_0 = \infty) \\
     & \ge 0.0001.
    \end{align*}

\end{proof}


\section{Connection to Generalized Linear Bandit Feedback}
\label{subsec:db2glm}

As observed in \cite{saha21}, we note that the relation of our preference feedback model to that of generalized linear model (GLM) based bandits \cite{filippi10,li17}--precisely the feedback mechanism. The setup of \emph{GLM bandits} generalizes the stochastic \emph{linear bandits} problem \cite{dani08,Yadkori11}, where at each round $t$ the learner is supposed to play a decision point $\x_t$ from a set fixed decision set $\cD \subset \R^d$, upon which a noisy reward feedback $f_t$ is revealed by the environment such that 
$ 
f_t = \mu(\x_t^\top\boldsymbol{\theta}) + \nu'_t,
$ 
$\boldsymbol{\theta} \in \R^d$ being some unknown fixed direction, $\mu: \R \mapsto \R$ is a fixed strictly increasing link function, and $\nu'_t$ is a zero mean $s$ sub-Gaussian noise for some universal constant $s >0$, i.e. $\E\big[ e^
{\lambda\nu'_t } \mid \cH_t \big] \le e^{\frac{\lambda^2s^2}{2}}$ and $\E[\nu'_t \mid \cH_t] = 0$ (here $\cH_t$ denotes the sigma algebra generated by the history $\{(x_\tau,o_\tau)\}_{\tau=1}^{t}$ till time $t$). 

The important connection now to make is that our structured dueling bandit feedback can be modeled as a GLM-Bandit feedback model on the decision space of pairwise differences $\cD' := \{(\x-\y) \mid \x,\y \in \cD \}$, since in this case the feedback received by the learner upon playing a duel $(\a_t ,\b_t)$ can be seen as:
$ 
o_t = \sigma\big( \z_t^\top\w^* \big) + \nu_t
$ 
where we denote by $\z_t = (\a_t-\b_t)$ and $\nu_t$ is a $0$-mean $\cH_t$-measurable random binary noise such that 
\begin{align*}
\nu_t = 
\begin{cases}
1-\sigma\big( \z_t^\top\w^* \big), \text{ with probability } \sigma\big( \z_t^\top\w^* \big),\\
-\sigma\big( \z_t^\top\w^* \big), \text{ with probability } \big(1-\sigma\big( \z_t^\top\w^* \big)\big),
\end{cases} 
\end{align*}
where we denote $\z_t := (\x_t-\y_t) \in \cD'$. Further since $\sigma: \cD' \mapsto [0,1]$, it is easy to verify that $\nu_t$ is $\frac{1}{2}$ sub-Gaussian. Thus our dueling based preference feedback model can be seen as a special case of GLM bandit feedback on the decision space $\cD'$ and the link function $\mu(\cdot)$ in our case is the sigmoid $\sigma(\cdot)$.

The above connection is crucially used in our proposed algorithm given in \cref{sec:alg_gen}, towards estimating the unknown model parameter $\w$, see \cref{lem:glm_ucb}.  

\section{Appendix for \cref{sec:gen}}
\label{app:gen}


\subsection{Proof of Lemmas for \cref{thm:reg_gen}}

\gencon*

\begin{proof}
Our analysis is inspired from the MLE concentration techniques used in GLM-bandit literature \citep{filippi2010parametric,li17}.

We start by noting that the maximum-likelihood estimation can be written as the solution to the following equation

\begin{align}
\label{eq:score}
    & \sum_{(\x,\y) \in \cC_\ell}\hat o_t (\a_t-\b_t)
    = \sum_{t \in \Tau_\ell} \sigma\big( (\a_t-\b_t)^\top \hw_\ell  \big)(\a_t-\b_t) \nonumber
    \\
    \implies
    &
    \sum_{(\x,\y)\in \cC_{\ell}} \varepsilon_{xy}(\x-\y)
    +
    \sum_{t \in \Tau_\ell} o_t(\a_t-\b_t)
    = \sum_{t \in \Tau_\ell} \sigma\big( (\a_t-\b_t)^\top \hw_\ell  \big)(\a_t-\b_t) 
\end{align}
Define any mapping $G_\ell:\R^d \mapsto \R$, s.t.  
$G_\ell(\w):= \sum_{t \in \Tau_\ell} \left(\sigma(\z_t^\top\w)-\sigma(\z_t^\top\w^*)\right) \z_t$.
Note we have:  
\begin{equation} 
\label{eq:mle}
G_\ell(\w^*)=0 \;\; \text{and }\; G_\ell(\hw_\ell)= \sum_{(\x,\y)\in \cC_{\ell}} \varepsilon_{xy}(\x-\y) + \sum_{t \in \Tau_\ell}\nu_t \z_t \,,
\end{equation} 
where the noise $\nu_\ell$ is as justified in \cref{subsec:db2glm} with sub-Gaussianity parameter $1/2$. For convenience, define $Z = G_\ell(\hw_\ell)$. 
We further define $\Delta_\ell:= \hat{\w}_\ell-\w^*$. Using Taylor series we have:
\begin{equation} 
\label{eq:taylor1}
G_\ell(\w_1) - G_\ell(\w_2) =\left [\sum_{t \in \Tau_\ell}\dot{\sigma}(\z_t^\top \bar{\w})\z_t\z_t^\top \right](\w_1-\w_2):=F(\bar{\w})(\w_1-\w_2),
\end{equation}
where $\dot{\sigma}(\cdot)$ denotes the first derivative of ${\sigma}(\cdot)$, we define $F: \R^d \mapsto \R$ s.t. $F(\w):= \left [\sum_{t \in \Tau_\ell}\dot{\sigma}(\z_t^\top {\w})\z_t\z_t^\top \right]$ and $\bar \w = v \w_1 + (1-v) \w_2$ for some $v \in (0,1)$. 
This it follows from \eqref{eq:taylor1} that there exists a $v \in [0,1]$ such that:
\[
Z=G_\ell(\hat{\w_\ell})-G_\ell(\w^*)=(H+E)\Delta\,,
\] 
where $\tilde{\w} = v \w^* + (1-v)\hat{\w}$, $H:= F(\w^*):= \left [\sum_{t \in \Tau_\ell}\dot{\sigma}(\z_t^\top {\w^*})\z_t\z_t^\top \right]$ and $E = F(\tilde{\w})-F(\w^*)$. 

Again applying mean value theorem, 
\begin{eqnarray*}
E = \sum_{t \in \Tau_\ell} \left(\dot{\sigma}(\z_t^\top \tilde{\w}) - \dot{\sigma}(\z_t^\top{\w^*})\right)\z_t\z_t^\top= \sum_{t \in \Tau_\ell} \ddot{\sigma}(r_i)\z_t^\top \Delta_\ell \z_t\z_t^\top
\end{eqnarray*}
for some $r_i \in [\z_t^\top{\tilde \w},\z_t^\top{\w^*}]$.
Now we are ready to prove the theorem. For any $\x \in \cD$, 
\begin{equation} 
\x^\top_1 ({\hw}_\ell-\w^*) \,=\, \x^\top_1 (H+E)^{-1} Z \,=\, \x^\top_1 H^{-1} Z - \x^\top_1 H^{-1} E (H+E)^{-1} Z\,.
\label{eqn:decomposition}
\end{equation}
Note that the matrix $(H+E)$ is nonsingular, so its inversion exists. Note also that for us, we have that $\dot{\sigma}(t) \le 1/4$ and $\ddot{\sigma}(t) \le  M_\sigma  = 1/4$ for all $t$~(see Assumption 2 of \cite{li17}). Moreover, we have that $H \succcurlyeq \kappa V_\ell$ by definition 
of $\kappa$ (recall we have $\kappa := \inf_{\{\norm{\x} \in \cB_d(1),\,\,\norm{\w-\w^*} \le 1\}} \dot{\sigma}(\x^\top\theta)  > 0 $). 

We will begin by upper bounding the first term in~\eqref{eqn:decomposition}. We have that 
\begin{align*}
\x^\top_1 H^{-1} Z 
    & =    \sum_{(\x,\y)\in \cC_{\ell}} \varepsilon_{xy}\x^\top_1 H^{-1}(\x-\y) + \sum_{t \in \Tau_\ell}\nu_t \x^\top_1 H^{-1} \z_t \\
\end{align*}
Now note that for the privacy term $\varepsilon_{xy}$,  we have that for all $\x,\y$
\begin{align*}
    |\x^\top_1 H^{-1} (\x-\y) |
    & \le \max_{(\x,\y) \in \cC_{\ell} } 2 |\x^\top_1 H^{-1} \z_{\x,\y} | \le \frac{2}{\kappa} \max_{(\x,\y) \in \cC_{\ell} } \norm{\z_{\x,\y}}_{V_\ell^{-1}}^2.
\end{align*}
Moreover, since $\varepsilon_{xy} \sim \mathsf{Lap}(\log T/\eps)$, standard concentration bounds for Laplace distributions imply that for all $\x,\y$ and all phases, with probability $1-\delta$
\begin{equation*}
    |\varepsilon_{xy}| \le 2 \cdot \log(d/\delta) \log(T)/\epsilon,
\end{equation*}
Overall this implies that with probability $1-\delta$, the privacy term is upper bounded by 
\begin{equation*}
    \sum_{(\x,\y)\in \cC_{\ell}} \varepsilon_{xy}\x^\top_1 H^{-1}(\x-\y)
    \le |\cC_{\ell} |^2 \cdot \frac{4\log(d/\delta) \log(T)}{\kappa \epsilon} \max_{(\x,\y) \in \cC_{\ell} } \norm{\z_{\x,\y}}_{V_\ell^{-1}}^2.
\end{equation*}
For the non-privacy term, the same arguments as~\cite[Inequality 23]{li17} show that with probability $1-\delta$
\begin{equation*}
     \sum_{t \in \Tau_\ell}\nu_t \x^\top_1 H^{-1} \z_t \le \frac{4 \gamma \sqrt{\log(T/\delta)}}{\kappa} \norm{\x_1}_{V_\ell^{-1}}.
\end{equation*}
Thus, we proved that
\begin{align*}
    |\x^\top_1 H^{-1} (x-y) |
    \le  |\cC_{\ell} |^2 \cdot \frac{4\log(d/\delta)\log(T)}{\kappa \epsilon} \max_{(\x,\y) \in \cC_{\ell} } \norm{\z_{\x,\y}}_{V_\ell^{-1}}^2 +  \frac{4 \gamma \sqrt{\log(T/\delta)}}{\kappa} \norm{\x_1}_{V_\ell^{-1}}.
\end{align*}

We now move to the second term in inequality~\eqref{eqn:decomposition}. We again split to the non-private and private parts, and handle each one separately.
We  have
\begin{equation*} 
\x^\top_1 H^{-1} E (H+E)^{-1} Z\,
=  \x^\top_1 H^{-1} E (H+E)^{-1} \left( \sum_{(\x,\y)\in \cC_{\ell}} \varepsilon_{xy}(\x-\y) + \sum_{t \in \Tau_\ell}\nu_t  \z_t \right)
\end{equation*}
For the non-private term, the arguments of~\cite[Inequality 25]{li17} show that with probability $1-\delta$
\begin{equation*} 
 \x^\top_1 H^{-1} E (H+E)^{-1} \sum_{t \in \Tau_\ell}\nu_t  \z_t \le 32 M_\sigma \frac{\gamma^2}{\kappa^3} \frac{d + \log(T/\delta)}{\sqrt{\lambda_{\min}(V_\ell)}} \norm{\x_1}_{V_\ell^{-1}}.
\end{equation*}
Moreover, for the privacy part, we have that all $x,y \in \cC_{\ell}$, 
\begin{align*}
|\varepsilon_{xy}& \x^\top_1 H^{-1} E (H+E)^{-1} (\x-\y)|
     \le |\varepsilon_{xy}|\norm_{\x_1}_{H^{-1}} \norm{H^{-1/2} E (H+E)^{-1} (\x -\y)} \\
    & \le |\varepsilon_{xy}|\norm_{\x_1}_{H^{-1}} \norm{H^{-1/2} E (H+E)^{-1} H^{1/2} } \norm_{\x -\y}_{H^{-1}}  \\
    & \le |\varepsilon_{xy}| \frac{1 }{\kappa} \norm_{\x_1}_{V_\ell^{-1}} \norm{H^{-1/2} E (H+E)^{-1} H^{1/2} } \norm_{\x -\y}_{V_\ell^{-1}}  \\
    & \stackrel{(i)}{\le}  |\varepsilon_{xy}| \frac{8 M_\sigma \gamma }{\kappa^3} \cdot \sqrt{\frac{d + \log(T/\delta)}{\lambda_{\min}(V_\ell)}} \cdot \norm_{\x_1}_{V_\ell^{-1}} \norm_{\x -\y}_{V_\ell^{-1}} \\
    & \le \frac{2 \log(d/\delta) \log(T)}{\epsilon} \cdot   \frac{16 M_\sigma \gamma }{\kappa^3} \cdot \sqrt{\frac{d + \log(T/\delta)}{\lambda_{\min}(V_\ell)}} \cdot \max_{(\x,\y) \in \cC_{\ell} } \norm{\z_{\x,\y}}_{V_\ell^{-1}}^2,
\end{align*}
where the $(i)$ follows since $\norm{H^{-1/2} E (H+E)^{-1} H^{1/2} } \le 8 M_\sigma \cdot \frac{\gamma}{\kappa^2} \sqrt{\frac{d + \log(T/\delta)}{\lambda_{\min}(V_\ell)}}$~\cite[Inequality 25]{li17}. Overall we now have
\begin{align*}
   \x^\top_1 ({\hw}_\ell-\w^*) 
   & \le   |\cC_{\ell} |^2 \cdot \frac{4\log(d/\delta)\log(T)}{\kappa \epsilon} \max_{(\x,\y) \in \cC_{\ell} } \norm{\z_{\x,\y}}_{V_\ell^{-1}}^2 +  \frac{4 \gamma \sqrt{\log(T/\delta)}}{\kappa} \norm{\x_1}_{V_\ell^{-1}} 
    \\
    & \quad + 32 M_\sigma \cdot \frac{\gamma}{\kappa^3} \frac{\sqrt{d + \log(T/\delta)}}{\sqrt{\lambda_{\min}(V_\ell)}} \cdot   \frac{\log(d/\delta)\log(T)}{\epsilon} |\cC_{\ell} |^2  \max_{(\x,\y) \in \cC_{\ell} } \norm{\z_{\x,\y}}_{V_\ell^{-1}}^2  \\
    & \quad +  32 M_\sigma \frac{\gamma^2}{\kappa^3} \frac{d + \log(T/\delta)}{\sqrt{\lambda_{\min}(V_\ell)}} \norm{\x_1}_{V_\ell^{-1}}.
\end{align*}

Now note that for any epoch $\ell = 1, 2 \ldots$, by \citep[Proposition 1]{li17} one knows that for the choice of 
  $t_0 = 2\Bigg( \frac{C_1 \sqrt d + C_2 \sqrt{\log (1/\delta)}}{\lambda_{\min}(V(\pi_{1}))} \Bigg)^2 + \frac{2\Lambda}{\lambda_{\min}(V(\pi_{1}))}$, such that $V(\pi_1) = \E_{\z_{\x,\y} \overset{\text{iid}}{\sim} \pi_{1}}[(\x-\y)(\x-\y)^{\top}]$, and defining $V_{0,\ell} = \sum_{t = t_\ell}^{t_\ell + t_0} (\a_t - \b_t)(\a_t - \b_t)^\top$, we have $\lambda_{\min}(V_{0,\ell}) \geq \Lambda$. 
  
  However, for our choice of $V(\pi_{1}) = \E_{\z_{\x,\y} \overset{\text{iid}}{\sim} \pi_{1}}[(\x-\y)(\x-\y)^{\top}]$ and the definition of G-Optimal design (\cref{defn:gopt}) we have 
  \[
  g(\pi_1) = \max_{\x \in \cB_d(1)} \norm{\x}^2_{V(\pi_{1})^{-1}} 
  =\lambda_{\max}(V(\pi_{1})^{-1}).
  \]

  But this gives $\lambda_{\min}(V(\pi_{1})) = \frac{1}{\lambda_{\max}(V(\pi_{1})^{-1})} = \frac{1}{g(\pi_1)} = \frac{1}{d}$, where the last equality follows from Kiefer–Wolfowitz theorem. 
  
  This implies $\lambda_{\min}(V(\pi_{1})) = 1/d$. Using this in the expression of $t_0$ and setting $\Lambda = \frac{8}{\kappa^4} \left(d^2 + \log\frac{T}{\delta}\right)$, we get for $t_0 = 2\Bigg( {C_1 d \sqrt d + C_2 d \sqrt{\log (1/\delta)}} \Bigg)^2 + { \frac{16d }{\kappa^4} \left(d^2 + \log\frac{T}{\delta}\right)}$, we have $\lambda_{\min}(V_\ell) \geq \lambda_{\min}(V_{0,\ell}) \geq \frac{8}{\kappa^4} \left(d^2 + \log\frac{T}{\delta}\right)$. 
Thus noting
\begin{equation}
\sqrt{\lambda_{\min}(V_\ell)} \ge  \frac{\sqrt{8} }{\kappa^2} \sqrt{\left(d^2 + \log\frac{T}{\delta}\right)}, 
\label{eq:lmin1}
\end{equation}
we further get:
\begin{align*}
   \x^\top_1 ({\hw}_\ell-\w^*) 
   & \le |\cC_{\ell} |^2 \cdot \frac{4\log(d/\delta)\log(T)}{\kappa \epsilon} \max_{(\x,\y) \in \cC_{\ell} } \norm{\z_{\x,\y}}_{V_\ell^{-1}}^2 +  \frac{4 \gamma \sqrt{\log(T/\delta)}}{\kappa} \norm{\x_1}_{V_\ell^{-1}} 
    \\
    & \quad + \frac{32 M_\sigma \gamma}{\kappa} \cdot \left(    \frac{\log(d/\delta) \log(T)}{\epsilon} |\cC_{\ell} |^2  \max_{(\x,\y) \in \cC_{\ell} } \norm{\z_{\x,\y}}_{V_\ell^{-1}}^2
   + \gamma \sqrt{\log(T/\delta)} \norm{\x_1}_{V_\ell^{-1}} \right)
   \\
   & \le  (32M_\sigma \gamma + 4) \cdot \left(|\cC_{\ell} |^2 \frac{\log(d/\delta) \log(T)}{\kappa \epsilon} \max_{(\x,\y) \in \cC_{\ell} } \norm{\z_{\x,\y}}_{V_\ell^{-1}}^2 +  \frac{\gamma \sqrt{\log(T/\delta)}}{\kappa} \norm{\x_1}_{V_\ell^{-1}} \right).
\end{align*}
Noting that $M_\sigma = 1/4$, $\gamma = 1/2$, and $|\cC_{\ell} | \le d^2$, the claim follows.
\end{proof}

\corglmucb*

\begin{proof}
We start by noting that, at phase $\ell$, \cref{alg:gen} pulls any duel $(\x,\y) \in \cC_\ell$ for $T_\ell^{\x,\y}$ times, where
\[
T_\ell^{\x,\y}:= 
         \ceil{\frac{6Cd^5\log(dT/\delta)}{\kappa \epsilon \xi_\ell } + \frac{25  d C^2 {\log(T/\delta)}}{\kappa^2 {\xi_\ell}^2}}\pi_{\ell}(\z_{\x,\y})
\]
Using $N_1 = \ceil{\frac{25  d C^2 {\log(T/\delta)}}{\kappa^2 {\xi_\ell}^2}}$,
and
$N_2 = \ceil{\frac{6Cd^5\log(dT/\delta)}{\kappa \epsilon \xi_\ell }}$, we can further write 
\[
T_\ell^{\x,\y} = (N_1 + N_2)\pi_{\ell}(\z_{\x,\y}).
\]
Now by design, we have 

\begin{align*}
    V_\ell = \sum_{\x,\y} T^{\x,\y}_{\ell} (\x-\y) (\x-\y)^\top \leq  (N_1+N_2)\sum_{\x,\y} \pi_{\ell}(\z_{\x,\y}) (\x-\y) (\x-\y)^\top =  (N_1+N_2)V(\pi_\ell),
\end{align*}
where $V(\pi_\ell)$ is as defined in \cref{defn:gopt}. Further using \emph{Kiefer–Wolfowitz Theorem} (\cref{sec:prelims}), we have $g(\pi_\ell) = d$; or in other words 
$\max_{(\x,\y) \in \cC_{\ell} } \norm{\z_{\x,\y}}_{V_\ell^{-1}}^2 = d$. 
But this further implies, 
\begin{align}
\label{eq:tmp1}
\max_{(\x,\y) \in \cC_{\ell} } \norm{\z_{\x,\y}}_{V_\ell^{-1}}^2
\leq \frac{1}{(N_1+N_2)} \max_{(\x,\y) \in \cD_{\ell}^2 } \norm{\z_{\x,\y}}_{V(\pi_\ell)^{-1}}^2 \le \nicefrac{d}{(N_1+N_2)}
\end{align} 

Similarly for any $\x \in \cD_\ell$, 
\begin{align}
\label{eq:tmp2}
 &\norm{\x}^2_{V_\ell^{-1}} \le \max_{\z \in \cD_\ell^2}\norm{\z}^2_{V_\ell^{-1}} \leq \frac{1}{(N_1+N_2)}\max_{\z \in \cD_\ell^2}  \norm{\z}^2_{V(\pi_\ell)^{-1}} \le \nicefrac{d}{(N_1+N_2)}  \nonumber 
\\
\implies & 
\norm{\x}^2_{V_\ell^{-1}} \leq \sqrt{\nicefrac{d}{(N_1+N_2)}}. 
\end{align} 

Now from \cref{lem:glm_ucb}, we know for any $\x \in \cD_\ell \subset \cD$:

\begin{align*}
   \abs{\x^\top ({\hw}_\ell-\w^*)} 
   & \le
  8\left(\frac{d^4\log(dT/\delta)}{\kappa \epsilon} \max_{(\x,\y) \in \cC_{\ell} } \norm{\z_{\x,\y}}_{V_\ell^{-1}}^2 +  \frac{\gamma \sqrt{\log(T/\delta)}}{\kappa} \norm{\x}_{V_\ell^{-1}}\right)
\end{align*}

Let $N = N_1 + N_2$. Further using \eqref{eq:tmp1}, we see that:
\begin{align*}
    \frac{8d^4\log(dT/\delta)}{\kappa \epsilon} \max_{(\x,\y) \in \cC_{\ell} } \norm{\z_{\x,\y}}_{V_\ell^{-1}}^2 \leq \frac{8d^5\log(dT/\delta)}{\kappa \epsilon N}  \leq \frac{2^{-\ell}}{2}, \text{ whenever }   N \geq \frac{16d^5\log(dT/\delta)}{\kappa \epsilon \xi_\ell }
\end{align*}

On the other hand, noting $\gamma = 1/2$ and using \eqref{eq:tmp2}, we have
\begin{align*}
     \frac{8\gamma \sqrt{\log(T/\delta)}}{\kappa} \norm{\x}_{V_\ell^{-1}} \leq \frac{ 8\sqrt d \gamma \sqrt{\log(T/\delta)}}{\kappa \sqrt N} \leq \frac{2^{-\ell}}{2}, \text{ whenever }
     N \geq \frac{64  d  {\log(T/\delta)}}{\kappa^2 {\xi_\ell}^2}.
\end{align*}
The claim follows noting our choice of $N_1$ and $N_2$ satisfies $N = (N_1 + N_2) \geq \max\biggn{\frac{16d^5\log(dT/\delta)}{\kappa \epsilon \xi_\ell }, \frac{64  d C^2 {\log(T/\delta)}}{\kappa^2 {\xi_\ell}^2}}$. 
\end{proof}

\epchlen*

\begin{proof}
$t_0$ is owing to the initial exploration at the beginning of each phase $\ell$. 
    The claim follows from our pairwise action selection rule, which pulls any pair $(\x,\y) \in \cC_\ell$ for 
    \[
    T_\ell^{\x,\y}:= 
         \ceil{\frac{16d^5\log(dT/\delta)}{\kappa \epsilon \xi_\ell } + \frac{64  d {\log(T/\delta)}}{\kappa^2 {\xi_\ell}^2}}\pi_\ell(\z_{\x,\y})
         \]
         times. We can derive the length of the phase $\ell$ by $\sum_{(\x,\y) \in \cC_\ell}T_\ell^{\x,\y}$ and the fact that $\sum_{(\x,\y) \in \cC_\ell}\pi_\ell(\z_{\x,\y}) = 1$.
\end{proof}

\beststays*

\begin{proof}
Given any $\x \in \cD_\ell$,
from \cref{cor:glm_ucb} we get that with probability at least $(1-3\delta)$, 
\begin{align}
\label{eq:rrecon1}
|\x^\top(\hw_\ell - \w^*)| \le \xi_\ell= 2^{-\ell}.
\end{align}    

Noting $B \geq |\cD_\ell|$, by taking a union bound over all $\ell$ and all $\x \in \cD_\ell$, we get that 
\begin{align}
\label{eq:conc1}
Pr(\exists \ell, \text{ and } \x \in \cD_\ell \mid |\x^\top(\hw_\ell - \w^*)| > \xi_\ell) \leq 3B \ell\delta.    
\end{align}

However, the item $\x^*$ can only be eliminated in phase $\ell$ only if \begin{align*}
    & \exists \y \in \cD_\ell, \text{ s.t. } \hw^\top_\ell(\x^*-\y) + 2\xi_\ell < 0
    \\
    & \exists \y \in \cD_\ell, \text{ s.t. } \hw^\top_\ell\x^* + \xi_\ell < \hw^\top_\ell \y - \xi_\ell,
\end{align*}
leading to a contradiction under the above concentration bound in \eqref{eq:conc1}. This is since by definition of $\x^*$ and \eqref{eq:conc1} we get with high probability $(1-3\delta B \ell)$:
\begin{align*}
    \hw^\top_\ell \y - \xi_\ell \leq \w^*\y < \w^*\x^* \leq \hw^\top_\ell\x^* + \xi_\ell.
\end{align*}
Thus $\x^*$ is never eliminated with probability $3\delta B \ell$. This proves the result.
\end{proof}

\badout*

\begin{proof}
    The proof again makes use of the concentration inequality of \eqref{eq:conc1} as obtained from \cref{cor:glm_ucb} above. Moreover, by \cref{lem:beststays}, since $x^*$ always belong to any phase $\ell$, we get with high probability $\x^*$ will eliminate any item $\x \in \cD_{\ell_x}$ at phase $\ell_x$ as: 
    \begin{align*}
    \hw^\top_\ell \x + \xi_\ell = \w^{*\top} \x + 2\xi_\ell < \w^*\x^* \leq \hw^\top_\ell\x^* + \xi_\ell,
\end{align*}
which satisfies the elimination criterion of $\x$.
    The claim hence follows.
\end{proof}

\end{document}